%% file: arxiv.tex
\documentclass[twoside]{article}
\usepackage[accepted]{aistats2018}

\usepackage{natbib}
\renewcommand{\cite}{\citep}
\usepackage{textcomp}
\usepackage{bm}             
\usepackage{amsthm}         
\usepackage{amsmath}        
\usepackage{amsfonts}       

\usepackage[breaklinks]{hyperref}     
\usepackage{subcaption}
\usepackage{graphicx}
\graphicspath{{figures/}}
\usepackage{pdfpages}       
\usepackage{array}

\newcommand\Mark[1]{\textsuperscript#1}
\newtheorem{theorem}{Theorem}
\newtheorem{lemma}{Lemma}
\begin{document} 

\runningtitle{Bayesian Nonparametric Poisson Process Allocation for Time-Sequence Modeling}
\runningauthor{Hongyi Ding, Mohammad Emtiyaz Khan, Issei Sato, Masashi Sugiyama}

\twocolumn[

\aistatstitle{Bayesian Nonparametric Poisson-Process Allocation for Time-Sequence Modeling}

\aistatsauthor{
	Hongyi Ding\Mark{*} \\
	\And\quad
	Mohammad Emtiyaz Khan\Mark{\textdied}\\
	\And\quad
	Issei Sato\Mark{*}\Mark{\textdied}\\
	\And
	Masashi Sugiyama\Mark{\textdied}\Mark{*}\\
}

\aistatsaddress{ \Mark{*}The University of Tokyo, Japan \And \Mark{\textdied}The RIKEN Center for AIP, Tokyo, Japan}
]




\begin{abstract} 
Analyzing the underlying structure of multiple time-sequences provides insights into the understanding of social networks and human activities. In this work, we present the \emph{Bayesian nonparametric Poisson process allocation}
(BaNPPA), a latent-function model for time-sequences, which  automatically infers the number of latent functions. We model the intensity of each sequence as an infinite mixture of latent functions, each of which is obtained using a function drawn from a Gaussian process. We show that a technical challenge for the inference of such mixture models is the unidentifiability of the weights of the latent functions. 
We propose to cope with the issue by regulating the volume of each latent function within a variational inference algorithm.
Our algorithm is computationally efficient and scales well to large data sets. We demonstrate the usefulness of our proposed model through experiments on both synthetic and real-world data sets.

\end{abstract} 

\section{Introduction}
\input{intro}

\section{Time-Sequence Modeling and Its Challenges}
\input{prelim}

\section{Bayesian Nonparametric Poisson Process Allocation (BaNPPA)}
\label{Sec:3}
\input{model}

\section{Inference}\label{Sec:4}
In this section, we first describe the general variational inference framework and provide a solution to the identifiability issue in Section \ref{Sec:Ident}. A derivation of the evidence lower bound (ELBO) and its derivatives are provided in the Appendix A.

\subsection{Variational Inference}
Denote $\bm{s}\stackrel{\Delta}{=}\{s_d\}$, $\Theta\stackrel{\Delta}{=}\{\theta_{dk}\}$ and $\bm{f}\stackrel{\Delta}{=}\{f_k\}$. Let $\bm{H}$ be the set of hyperparameters of the GP covariance function. The joint distribution of BaNPPA can be expressed as 
\begin{align*}
p(Y,\bm{\Theta,s,f})&=\prod_{d=1}^{D}p(\bm{y}_d|\bm{f,\theta_d},s_d)\prod_{d=1}^{D}\prod_{k=1}^\infty p(\theta_{dk};\alpha)\\
&\times\prod_{d=1}^{D}p(s_d;a_0,b_0)\prod_{k=1}^\infty p(f_k;g,\bm{H}).
\end{align*}
We approximate the posterior distribution over $\Theta$ and $\bm{f}$, while computing a point estimate of $\bm{s}$. We follow \citet{blei2006variational} to truncate the number of latent functions to $K$ which we select to be larger than the expected number of latent functions used by the data. For the GP part, we use the same set of pseudo inputs $\{t_m\}_{m=1}^M$, $M<N$ for each $f_k$ to reduce the number of variational parameters \cite{lloyd2016latent}.
Denote $\bm{f}_{k,M}$ to be the vector $[f_k(t_1),\ldots,f_k(t_M)]^\top$, $\kappa_{k,MM}$ to be a covariance matrix whose $i,j$'th entry is equal to $\kappa_k(t_i,t_j)$, and $\bm{g}_M\in\mathbb{R}^M$ to be a vector all of whose elements are equal to $g$. We choose the following forms for the variational distributions of $\theta_{dk}$ and $\bm{f}_{k,M}$:
\begin{align*}
q(\theta_{dk})&=\mathbb{I}(k< K)\mathrm{Gamma}(\tau_{dk,0},\tau_{dk,1})\\
&+\mathbb{I}(k= K)\delta(1)+\mathbb{I}(k> K)p(\theta_{dk}),\\
q(\bm{f}_{k,M})&=\mathbb{I}(k\leq K)
\mathcal{N}(\bm{\mu}_k,\Sigma_k)\\
&+\mathbb{I}(k> K)\mathcal{N}(\bm{g}_M,\kappa_{k,MM}),
\end{align*}
where $\mathbb{I}(\cdot)$ is the indicator function, $\delta(\cdot)$ is a dirac-delta function, and $\bm{\mu}_k$ and $\Sigma_k$ are the mean and covariance of a Gaussian distribution. Following \citet{lian2015multitask}, we use the re-parametrization $\Sigma_k = L_kL_k^T$ by Cholesky decomposition.

Using the approximation of \citet{titsias2009variational} and a mean-field assumption over $\Theta$, we can use the following final variational distribution: 
\begin{equation*}
   q(\bm{f}, \Theta) \stackrel{\Delta}{=} \prod_{k=1}^{\infty} p(\bm{f}_{k,N}|\bm{f}_{k,M})q(\bm{f}_{k,M})\prod_{d=1}^{D}\prod_{k=1}^\infty q(\theta_{dk}).
\end{equation*}
Denoting $\bm{\tau}\stackrel{\Delta}{=} \{(\tau_{dk,0},\tau_{dk,1})\}$ and $\bm{L}\stackrel{\Delta}{=}\{L_k\}$, we get the following set of variational parameters and hyperparameters to be optimized: $\Phi=\{\bm{\tau},\bm{\mu},\bm{L},\bm{H},a_0,b_0,\alpha, \bm{s}\}$. 
 
\subsection{An Alleviation Solution to the Identifiability Problem}
\label{Sec:Ident}
So far, the framework seems very traditional. However, as we mentioned in Section \ref{Sec:3}, this model has an additional identifiability problem which might make interpretability difficult. In this section, we propose a solution to alleviate this issue.

A straightforward option is to directly impose a constraint on the volume of the latent functions $\int_{\mathcal{T}}f_k^2(t)dt$, where $f_k$ is drawn from the posterior process $p(f_k|Y)$. However this is intractable. In order to obtain a tractable constraint, we could instead impose a constraint on the following expectation:
\begin{equation}
\iint_{\mathcal{T}}p(f_k|Y)f_k^2(s)dsdf_k=A,~k=1,\ldots,K,\label{Equ:ConstraintP}
\end{equation}
where $A$ is a positive constant. Within the variational inference framework, we use the variational distribution $q(f_k)$ to approximate the posterior $p(f_k|Y)$, and add the following constraint to each latent function:
\begin{equation}
\iint_{\mathcal{T}}q(f_k)f_k^2(s)dsdf_k=A,\quad k=1,\ldots,K\label{Equ:Constraint}.
\end{equation}
The above constraint can be easily computed unlike the volume constraint on the function $f_k$. In our experiments, we set $A = N/D$ where $N$ is the total number of events in the data and $D$ is the number of time- sequences in $Y$.

\subsection{Optimization with Equality Constraints}
Given the equality constraints in Equation \eqref {Equ:Constraint}, the optimization process can be formulated as follows, where we denote the ELBO as $\mathcal{L}_1(\Phi)$:
\begin{align}
\max_\Phi~\mathcal{L}_1(\Phi) &\quad\mathrm{s.t.} ~h_k(\Phi)=0,~k=1,\ldots,K, \label{Equ:My Problem}\\
h_k(\Phi)&=\int_{\mathcal{T}}\mathbb{E}_q[f_k^2(s)]ds-A.\nonumber
\end{align}
Problem \eqref{Equ:My Problem} is an optimization problem with equality constraints and we use the augmented Lagrangian method \cite{bertsekas2014constrained} to transform Problem \eqref{Equ:My Problem} into a series of related optimization problems indexed by $i$:
\begin{equation}
\max_\Phi~\mathcal{L}_1(\Phi)-\sum_{k=1}^{K}\Big(w_{ik}h_k(\Phi)+\frac{1}{2}v_{ik}h_k^2(\Phi)\Big),  
\label{Equ:Penalty Solution}
\end{equation}
where $\{w_{ik}\}$ is a bounded sequence and $\{v_{ik}\}$ is a non-negative monotonically-increasing sequence with respect to $i$. We denote this objective $L_{\bm{v_i}}(\Phi,\bm{w_i})$. For each optimization problem in Equation \eqref {Equ:Penalty Solution}, $L_{\bm{v_i}}(\Phi,\bm{w_i})$ is still upper bounded (a proof is given in Appendix A). Thus if we use coordinate ascent with respect to $\Phi$, the algorithm is guaranteed to arrive at a local maximum.

To set $\bm{v}_{ik}$ and $\bm{w}_{ik}$, we follow the suggestions from \citet{bertsekas2014constrained}, and set $\bm{v}_{i+1,k}=4\bm{v}_{ik}$ and $ \bm{w}_{i+1,k}=\bm{w}_{ik}+\bm{v}_{ik}\bm{h}_k(\Phi_i)$. We initialize $v_{1k}=4,w_{1k}=1,\forall k$.

\subsection{Computational Complexity}
\label{Sec:Complexity}
Optimization problems shown in Equation \eqref {Equ:Penalty Solution} are not significantly more expensive than the original optimization problem. Although in Equation \eqref {Equ:Penalty Solution}, we have to optimize additional parameters, the bottleneck is still the matrix-matrix multiplication in the evaluation of $q(\bm{f}_{k,N})$. For one iteration of the training procedure, the computational complexity is $\mathcal{O}(KNM^2)$ , which is the same as LPPA.

One might expect that the total computational complexity of our algorithm is worse than LPPA because we have to solve a sequence of problems. We find that ``warm starts" are very effective in improving the convergence of our algorithm \cite{bertsekas2014constrained}. Namely, we reuse the final value $\Phi_{i-1}$ of the previous optimization as the starting value for the $i$'th round and terminate the training process when the relative change in the likelihood is small.
In our experiments, we observed that the convergence of BaNPPA is rather fast and comparable to LPPA.

\section{Experiments}
\label{sec:expt}
In this section, we evaluate our proposed BaNPPA model and compare it with LPPA. To measure the effect of adding the constraint shown in Equation \eqref{Equ:Constraint}, we also compare to a variant of BaNPPA which does not contain any constraints. We call it BaNPPA with No Constraints, i.e., BaNPPA-NC. This gives us three methods to compare: LPPA, BaNPPA-NC, and BaNPPA. The code to reproduce our experiments can be found at
\url{github.com/Dinghy/BaNPPA}.

\begin{table}[t]
	\caption{Data sets used for the experiments. Here, $D$ is the number of time-sequences, $N_\mathrm{train}$ and $N_\mathrm{test}$ are the total number of events in the training and test set respectively, and $\mathcal{T}$ is the time window.}
	\centering
	\vspace*{-3mm}
	\begin{center}
		\begin{tabular}{lllll}
			{ Data set}     & { D}   & $N_\mathrm{train}$ & $N_\mathrm{test}$ & $\mathbf{\mathcal{T}}$ \\
			\hline\\
			Synthetic A & 200   &6,304 & 6,010  &  [0,60] \\
			Synthetic B & 250   &8,074 & 8,110  &  [0,80] \\
			Microblog     & 500  &44,628 & 44,352  &  [3,15] \\
			Citation     & 600  &106,113& 106,340 &  [0,20]
		\end{tabular}
	\end{center}
	\label{Table:Data set}
\end{table}

We test the three methods on two synthetic and two real-world data sets.
Table \ref{Table:Data set} summarizes the overall statistics and we give detailed information below.

\textbf{1) Synthetic A.} We sample 200 time-sequences from a mixture of 4 latent functions $\bar{f}(t)$ shown in the top plot of Figure \ref{fig:SynInt}. The intensity function is defined as follows: $\lambda_d(t)=s_d\sum_{k=1}^{4}\theta_{dk}\bar{f}(t)$, where $s_d\sim \mathrm{Gamma}(2,3)$ and $\bm{\theta}_d\sim\mathrm{Dir}(1.2,1,0.8,0.6)$. Here $\mathrm{Dir}(\cdot)$ denotes the Dirichlet distribution. More details on the data generation process can be found in Appendix C.

\textbf{2) Synthetic B.} This data set is similar to Synthetic A but there are 6 latent functions shown in the bottom plot of Figure \ref{fig:SynInt}. In Synthetic B data set, $s_d\sim \mathrm{Gamma}(2,3)$ and $\bm{\theta}_d\sim\mathrm{Dir}(1.2,1,0.8,0.6,0.5,0.5)$.

\begin{figure}[t]
	\begin{center}
		\includegraphics[width=\columnwidth]{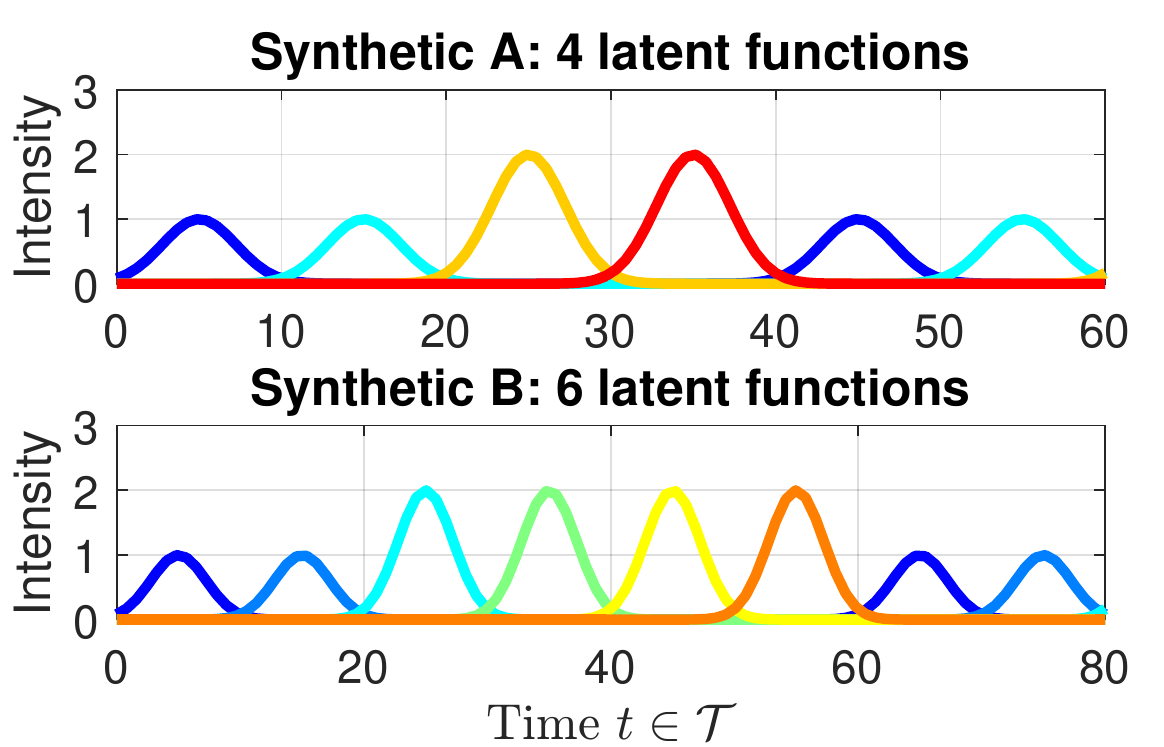}
		\caption{Latent functions used to create synthetic data set A and B are shown in the top and bottom plots, respectively. In both the data sets, there are two latent functions with two modes while the rest have only one mode.}\label{fig:SynInt}
	\end{center}
\end{figure}

\textbf{3) Microblog data set.} This data set contains 500 tweets and all retweets of each tweet from $7$ tweet-posters on Sina micro-blog platform obtained through the official API\footnote{\url{http://open.weibo.com/wiki/Oauth/en}}. Two examples are shown in Figure \ref{fig:EgMicroblog}. Through time-sequence modeling, we can try to understand the retweet patterns. For example, one reason could be that the tweets posted at an inactive hour (late at night) will regain the
attention from the followers several hours later next morning \cite{ding2015predicting,gao2015modeling}. BaNPPA could help us understand such reasons as illustrated in Figure \ref{fig:EgMicroblog}.

\textbf{4) Citation data set.} This data set contains the Microsoft academic graph until February 5th, 2016 obtained from the KDDcup 2016 \footnote{\url{https://kddcup2016.azurewebsites.net/}}. The original data set contains 126,909,021 papers and we use a subset of it.
Time-sequence modeling can be used to understand the patterns of citations, e.g., some papers quickly get citations while some others get it slowly. Two examples of this data set are given in the Appendix C.1.

\begin{figure*}[ht]
\centering
	\includegraphics[width=\textwidth]{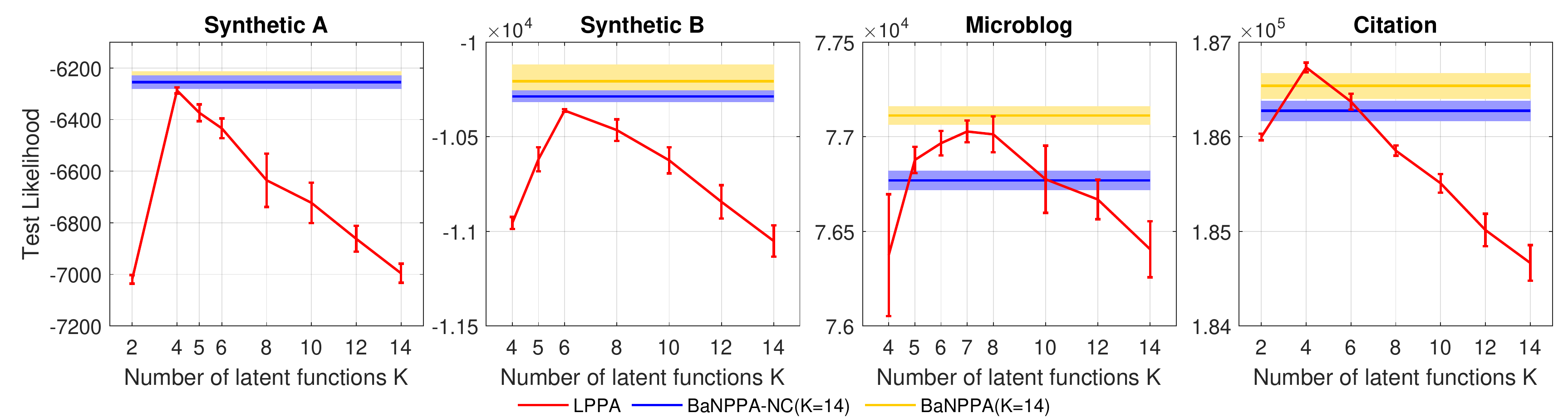}
	\caption{BaNPPA gives the best test-likelihoods (higher is better) and performs comparably to the best setting of $K$ for LPPA. For BaNPPA/BaNPPA-NC, we use a fixed value of $K=14$. Error bars and shaded areas show the 95\% confidence intervals.}
	\label{fig:Comparison}
\end{figure*}

\begin{figure*}[ht]
	\includegraphics[width=\textwidth]{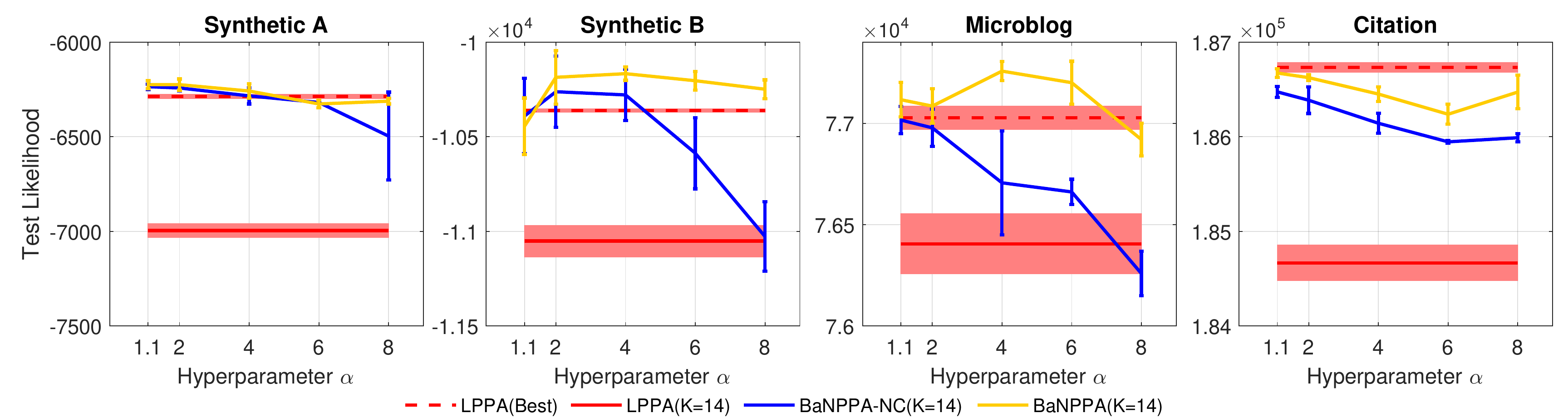}
	\caption{For a variety of hyperparameter values, BaNPPA gives the best performance which is also comparable to the best performance of LPPA and much better than LPPA with $K=14$. Performance of BaNPPA-NC degrades with increasing $\alpha$ while performance of BaNPPA is relatively stable.}
	\label{fig:Comparison_AlphaVary}
\end{figure*}

\textbf{Evaluation Metrics}. We evaluate the methods using the two metrics described below.

To measure the predictive performance, we follow \citet{lloyd2015variational} and use the following approximation to the test likelihood which we denote by $\mathcal{L}_{\mathrm{test}}(Y_{test}, \Theta, \bm{s}) = $
	\begin{align}
      &\sum_{d=1}^{D}\sum_{n=1}^{N^{d}_{\mathrm{test}}}\ln\Big( s_d\sum_{k=1}^{K}\theta_{dk}\exp\Big(\mathbb{E}_q (\ln f_k^2(t_n^d))\Big)\Big)\nonumber \\
	&-\sum_{d=1}^{D}s_d\sum_{k=1}^{K}\theta_{dk}\int_\mathcal{T}\mathbb{E}_q[f_k^2(s)]ds.
	\label{Equ:Test likelihood}
	\end{align} 
	This is a lower bound to the test likelihood $\ln p(Y_{\mathrm{test}}|Y_{\mathrm{train}})$ and a higher value means a better approximation of the test likelihood. For LPPA, the allocation parameters $\theta_{dk}$ are the point-estimated weights and the rate parameter $s_d=1$. For BaNPPA and BaNPPA-NC, we report averaged value over $q(\theta_d)$. We can compute a similar approximation $\mathcal{L}_{\mathrm{train}}$ on the training data. Detailed derivations and explanations are given in Appendix B.  
	
To measure the responsibility of each latent function in the model, we first define the normalized allocation matrix $\hat{\Theta}\in \mathbb{R}^{D\times K}_+$ whose $(d,k)$'th entry is equal to, 
\begin{equation}
	\hat{\theta}_{dk} =\frac{\mathbb{E}_q[\theta_{dk}\int_{\mathcal{T}}f^2_k(s)ds]}{\sum_{m=1}^{K}\mathbb{E}_q[\theta_{dm}\int_{\mathcal{T}}f^2_m(s)ds]}.
	\label{Equ:Occupancy}
\end{equation} 
The normalization in the above matrix tries to remove the unidentifiability introduced due to the unconstrained volume of the latent functions in LPPA and BaNPPA-NC.
%
Based on $\hat{\Theta}$, we can compute a normalized score that can measure the responsibility of each latent function. We define the normalized expected responsibility (NER) $\hat{\upsilon}_k=\sum_{d=1}^{D}\hat{\theta}_{dk}/D,~k=1,\ldots,K$. A larger NER indicates that the corresponding latent function is more often occupied by the model. Another measure is the unnormalized expected responsibility (UNER) $\upsilon_k=\sum_{d=1}^{D}\mathbb{E}_q[\theta_{dk}]/D,~k=1,\ldots,K$, which omits the
contribution of the volume of $f_k^2$.

\textbf{Experimental settings.} Our goal is to measure the improvements obtained with the automatic inference of $K$ using BaNPPA. To do so, we fix $K$ to $14$ for BaNPPA and BaNPPA-NC, and compare them to LPPA with a range of values for $K$. We expect BaNPPA to give a comparable performance to the best setting of $K$ in LPPA.

Choice of $K$ also affects hyperparameter estimation. To measure it, we conduct experiments for two different methods of setting the hyper-parameter $\alpha$. In the first method, we learn $\alpha$ within a variational framework (initialize $\alpha=1$, see details in Appendix A). In the second method, we do not learn $\alpha$ rather fix it to one of the value in the set $\{1.1,2,4,6,8\}$. In both methods, all experiments were repeated five times.
We use a random initialization for the allocation matrix $\Theta$ and $\bm{\tau}$. For sparse GPs, we use 18, 24, 30 and 30 pseudo inputs for the four data sets, respectively. We follow the common practice and add a jitter term $\varepsilon I$ to the covariance matrix $\kappa_{k,MM}$ to avoid numerical instability \cite{bauer2016understanding}. 
For hyper-parameters $a_0$ and $b_0$ in the gamma distribution, we use the counts of events $\{N^{d}_{\mathrm{train}}\}_{d=1}^D$ to initialize $(a_0,b_0)$. 

To maintain the positivity constraints on $\bm{L}$ and $\bm{\tau}$, we use the limited-memory projected quasi-Newton algorithm \cite{schmidt2009optimizing}. For BaNPPA, we stop the training process when the relative change between $L_{\bm{v_i}}(\Phi_i,\bm{w_i})$ and $L_{\bm{v_{i+1}}}(\Phi_{i+1},\bm{w_{i+1}})$ is less than $10^{-3}$. For other methods, we terminate the training process when the relative change in ELBO is less than $10^{-3}$.

\textbf{Performance Evaluation.} Figure \ref{fig:Comparison} shows a comparison of the test-likelihoods when optimizing $\alpha$, while Figure \ref{fig:Comparison_AlphaVary} shows the same when $\alpha$ is fixed.
In Figure \ref{fig:Comparison}, the test likelihood of LPPA drops when increasing the number of latent functions $K$. As desired, BaNPPA achieves comparable results to the best setting of $K$ in LPPA. BaNPPA-NC also performs well but slightly worse than BaNPPA.
In Figure \ref{fig:Comparison_AlphaVary}, when increasing $\alpha$, the performance of BaNPPA stays relatively stable and comparable to the best setting of LPPA. The performance of BaNPPA-NC however degrades with increasing $\alpha$ for all data sets. This shows that the volume constraint in BaNPPA improves the performance.
Other performance measures such as the training likelihood and computation time as well as the value of optimized $\alpha$ are given in Appendix C. 
\begin{figure}[ht]
	\begin{center}
		\includegraphics[width=\columnwidth]{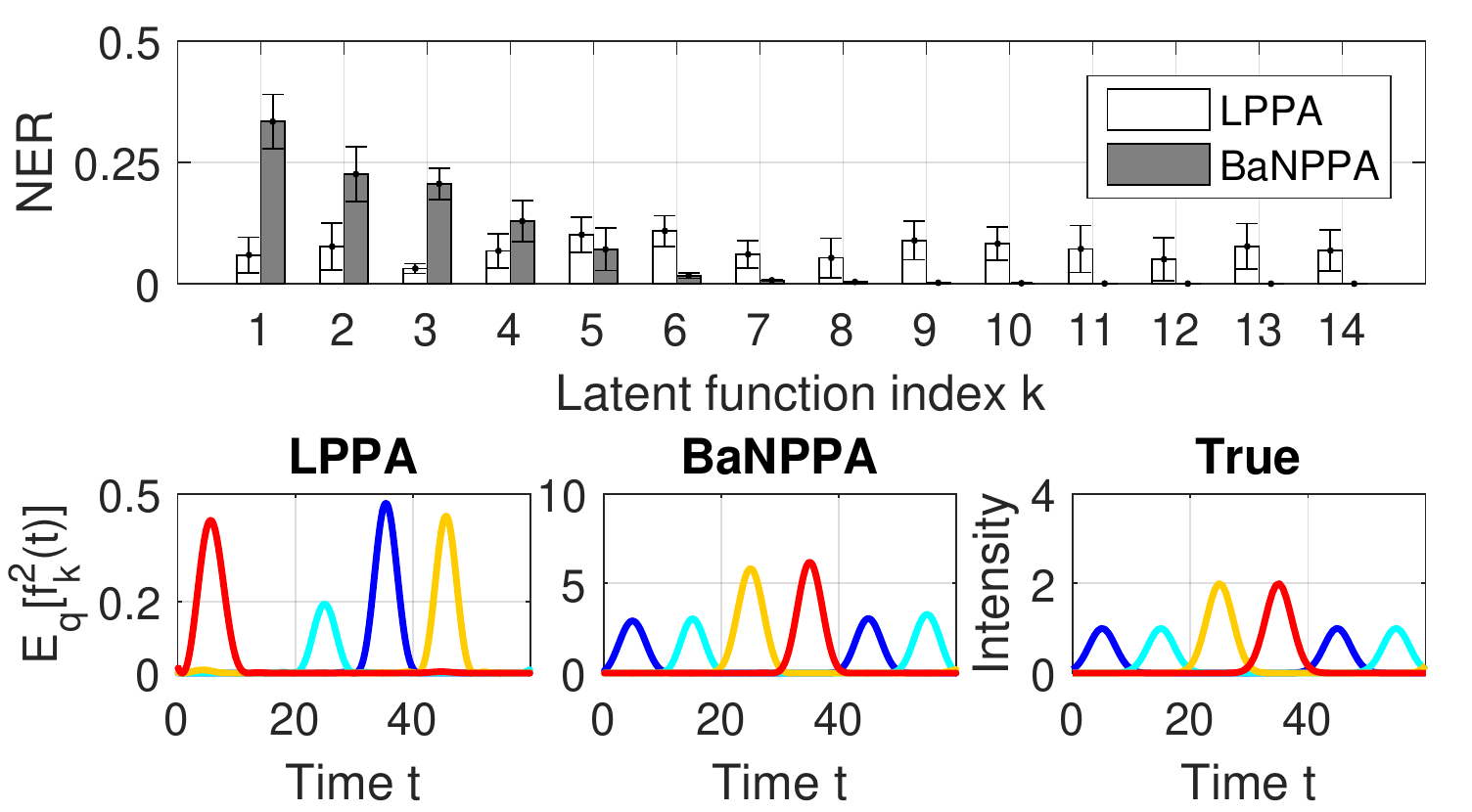}
		\caption{This figure shows that BaNPPA can reliably identify true latent functions for the Synthetic A data set. The top plot shows the NER scores for BaNPPA and LPPA for $K=14$ where we see that, under LPPA, all latent functions have nonzero NER, while, under BaNPPA, only a handful of them have significant NER scores. The bottom plot shows the top four latent functions (sorted according to NER) obtained for both the methods along with the true latent functions. We see that BaNPPA recovers functions very similar to the true functions.}\label{fig:NERBasis_Syn}
	\end{center}
\end{figure}
\begin{figure}[ht]
	\begin{center}
		\includegraphics[width=\columnwidth]{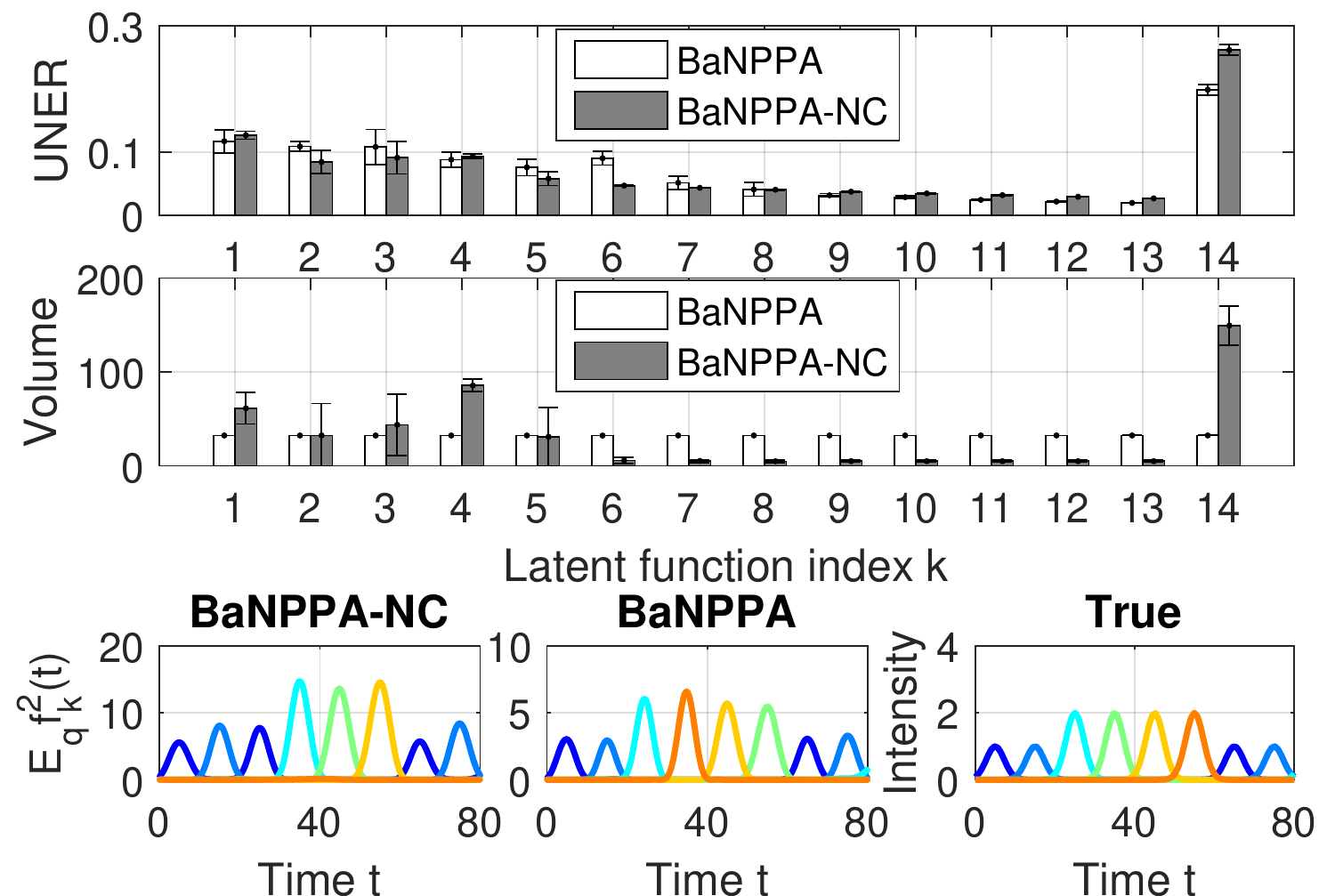}
		\caption{This figures shows that the volume constraint in BaNPPA is crucial to discover the true latent functions. Both BaNPPA and BaNPPA-NC obtain similar UNER score (top plot), yet the top latent functions obtained with the two methods are different (bottom plot). The imbalance in the volumes for BaNPPA-NC (middle plot) is the reason behind this difference. See the text for details.}\label{fig:UNERNERBasis_Syn}
	\end{center}
\end{figure}

For the Synthetic A data set, we further plot the NER scores (averaged over the five trials for $K=14$) in the top plot in Figure \ref{fig:NERBasis_Syn}. We see that, under LPPA, all latent functions have nonzero NER, while for BaNPPA only a small number of latent functions have high NER score.
In the bottom plot in Figure \ref{fig:NERBasis_Syn}, we show the top four latent functions sorted according to the NER scores. For these plots, we used the best runs shown in Figure \ref{fig:Comparison}. We see that LPPA does not recover the true latent functions, while BaNPPA gives very similar results to the truth.


To visualize the responsibilities further, we plot the NER score and the normalized allocation matrix $\hat{\Theta}$ for the Mircoblog dataset in Figure \ref{fig:MatrixPlot}. We show results for LPPA and BaNPPA. We choose runs that obtained the best test-likelihood in Figure \ref{fig:Comparison}, and visualize 100 time-sequences sampled randomly.
We see that as expected LPPA uses all latent functions to explain the data, while BaNPPA assigns almost zero weights to latent functions with higher indices.
This further confirms that, even when a large number of latent functions are given, BaNPPA automatically selects only a few to explain the data, while LPPA might overfit. 

Finally, we further explore the impact of the volume constraint of Equation \eqref{Equ:Constraint} in BaNPPA.
We compare BaNPPA and BaNPPA-NC on the Synthetic B data set in Figure \ref{fig:UNERNERBasis_Syn}.
We use results for $K=14$ and $\alpha=8$.
In the top plot, we see that BaNPPA and BaNPPA-NC both give similar UNER scores, yet as shown in the bottom plot, BaNPPA-NC does not recover the true latent functions.
This result can be explained by looking at the expected volume $\mathbb{E}_q[\int_{\mathcal{T}}f_k^2(t)dt]$ shown in the middle plot. For BaNPPA, the volumes of all latent functions are equal, while, for BaNPPA-NC, the latent functions with higher UNER scores are assigned higher volume which eventually also get higher weights. This imbalance in the weights for some functions makes the results of BaNPPA-NC and BaNPPA different from each other.
This result clearly shows that the volume constraint in BaNPPA plays an important role to recover the true latent functions which is important for interpretability.

%
%

Overall, BaNPPA-NC performs similarly to BaNPPA when the latent structure is simple but becomes less favorable when the structure gets complicated. We give three additional synthetic data experiments  in Appendix C where the true $K$ is large.

\section{Conclusions and Future Work}
We proposed a model for time-sequence data, called BaNPPA, to automatically infer the number of latent functions. We combined BNP methods with the existing LPPA method, and showed that this combination might result in undentifiability. We solve this problem by imposing a volume constraint within variational inference. In the future, we will consider further investigating the reasons behind the identifiability problem. We will also investigate ways to combine the volume constraint and the Gaussian process prior. 

\section*{Acknowledgements}

The authors would like to thank Wenzhao Lian for sharing the code and thank Chris Lloyd for helpful discussions. MS acknowledges support by KAKENHI 17H00757.

\bibliography{aistats2018}

\newpage
\appendix
\section{Evidence Lower Bound $\mathcal{L}_1(q)$}
Using Jensen's inequality, we bound the marginal log likelihood of the observed sequence $\{\bm{y}_d\}$. Hereafter we omit hyper-parameters $a_0,b_0,\alpha,\bm{H}$ in $\ln p(Y;a_0,b_0,\alpha,\bm{H})$ for simplicity.
\begin{align}
\ln p(Y)
&= \ln\Big[\int \Big(\prod_{d=1}^{D}p(\bm{y}_d|\theta_d,s_d,\bm{f})p(s_d)p(\theta_d')\Big)\nonumber\\
&\times\prod_{k=1}^{\infty}p(\bm{f}_{k,N}|\bm{f}_{k,M})p(\bm{f}_{k,M}) d\bm{\theta_d'} d\bm{f}\Big]\nonumber\\
&\geq \sum_{d=1}^{D}\mathbb{E}\ln p(\bm{y}_d|\theta_d,s_d,\bm{f})+\sum_{d=1}^{D}\sum_{k=1}^{K-1}\mathbb{E}\ln p(\theta_{dk}')\nonumber\\
&+\sum_{d=1}^{D}\mathbb{E}\ln p(s_d)+\sum_{k=1}^{K}\mathbb{E}\ln p(\bm{f}_{k,M})\nonumber\\
&-\sum_{d=1}^{D}\sum_{k=1}^{K-1}\mathbb{E}\ln q(\theta_{dk}') -\sum_{k=1}^{K}\mathbb{E}\ln q(\bm{f}_{k,M})\stackrel{\Delta}{=}\mathcal{L}_0(q).
\label{Equ:ELBO}
\end{align}
First we introduce a lemma \cite{paisley2010two}. 

\begin{lemma}{\cite{paisley2010two}}
	Let $\{X_k\}_{k=1}^K$ be a set of positive random variables, then
	\begin{equation}
	\mathbb{E}\ln\Big(\sum_{k=1}^{K}X_k\Big)\geq \ln\Big(\sum_{k=1}^{K}\exp(\mathbb{E}\ln X_k)\Big).
	\end{equation}
	or equivalently if $X_k = \exp(Y_k)$ where $Y_k$ is a random variable, then
	\begin{equation}
	\mathbb{E}\ln\Big(\sum_{k=1}^{K}\exp(Y_k)\Big)\geq \ln\Big(\sum_{k=1}^{K}\exp(\mathbb{E}Y_k)\Big).
	\end{equation}
	\begin{proof}
		The function $\ln(\cdot)$ is concave. Using an auxiliary probability vector, $(p_1,\ldots,p_K)$, where $p_k>0$ and $\sum_{k=1}^{K}p_k=1$, it follows from Jensen’s inequality that
		\begin{equation}
		\mathbb{E}\ln\Big(\sum_{k=1}^{K}X_k\Big)=\mathbb{E}\ln\Big(\sum_{k=1}^{K}p_k\frac{X_k}{p_k}\Big)\geq\sum_{k=1}^{K}p_k\mathbb{E}\ln\Big(\frac{X_k}{p_k}\Big)
		\end{equation}
		Taking derivatives with respect to $\{p_k\}$, we have
		\begin{equation}
		p_k = \frac{\exp(\mathbb{E}\ln X_k)}{\sum_{v=1}^{K}\exp(\mathbb{E}\ln X_v)}
		\end{equation}
		Inserting this back, we obtain the desired bound.
	\end{proof}
	\label{lem:1}
\end{lemma}

Using Lemma \ref{lem:1}, we could further bound the first term to allow for a practical variational inference. This result is the same as the one obtained by following the methodology in LPPA \cite{lloyd2016latent}.
\begin{align}
&\mathbb{E}\ln p(\bm{y}_d|\theta_d,s_d,\bm{f})\nonumber\\
&=\sum_{n=1}^{N_d}\Big(\ln\eta_d+\mathbb{E}\ln\sum_{k=1}^{\infty}\exp(\ln \theta_{dk}+\ln f_k^2(t))\Big)\nonumber\\
&-\eta_d\int_{\mathcal{T}}\mathbb{E}\sum_{k=1}^{\infty}\theta_{dk}f_k^2(s)ds\\
\geq&\sum_{n=1}^{N_d}\Big(\ln \eta_d+\ln\sum_{k=1}^{\infty}\exp(\mathbb{E}\ln \theta_{dk}+\mathbb{E}\ln f_k^2(t))\Big)\nonumber\\
&-\eta_d\int_{\mathcal{T}}\mathbb{E}\sum_{k=1}^{\infty}\theta_{dk}f_k^2(s)ds.\label{Equ:Collapse}
\end{align}
Using Equation \eqref{Equ:Collapse}, we implicitly collapse the indicator variables and obtain a lower bound of ELBO:
\begin{align}
\mathcal{L}_1(q)&\stackrel{\Delta}{=}\sum_{n=1}^{N_d}\Big(\ln \eta_d+\ln\sum_{k=1}^{\infty}\exp(\mathbb{E}\ln \theta_{dk}+\mathbb{E}\ln f_k^2(t))\Big)\nonumber\\
&-\eta_d\int_{\mathcal{T}}\mathbb{E}\sum_{k=1}^{\infty}\theta_{dk}f_k^2(s)ds+\sum_{d=1}^{D}\sum_{k=1}^{K-1}\mathbb{E}\ln 
\frac{p(\theta_{dk}')}{q(\theta_{dk}')}\nonumber \\
& +\sum_{d=1}^{D}\mathbb{E}\ln p(s_d)+\sum_{k=1}^{K}\mathbb{E}\ln \frac{p(\bm{f}_{k,M})}{q(\bm{f}_{k,M})}.
\label{Equ:ELBOalter}
\end{align}
Now $q(f_{k,N})=\mathcal{N}(\tilde{u}_k,\tilde{B}_k)$, where 
\begin{align*}
\tilde{u}_k &= \kappa_{k,NM}\kappa_{k,MM}^{-1}\mu_k,\\
\tilde{B}_k &=\kappa_{k,NN}-\kappa_{k,NM}\kappa_{k,MM}^{-1}\kappa_{k,MN}\\
&+\kappa_{k,NM}\kappa_{k,MM}^{-1}\Sigma_k\kappa_{k,MM}^{-1}\kappa_{k,MN}
\end{align*}
And the expectation parts in Equation \eqref{Equ:ELBOalter} can be computed as:
\begin{align}
&\mathbb{E}\ln p(\theta_{dk}')=\ln \alpha+(\alpha-1)\mathbb{E}[\ln(1-\theta_{dk}')],\\
&\mathbb{E}\ln q(\theta_{dk}')=\ln\frac{\Gamma(\tau_{dk,0}+\tau_{dk,1})}{\Gamma(\tau_{dk,0})\Gamma(\tau_{dk,1})}\nonumber\\
&\quad+(\tau_{dk,1}-1)\mathbb{E}[\ln(1-\theta_{dk}')]+(\tau_{dk,0}-1)\mathbb{E}[\ln\theta_{dk}'],\\
&\mathbb{E}\ln p(s_d) = a_0\ln b_0-\ln \Gamma(a_0)+(a_0-1)\ln \eta_d-b_0\eta_d,\\
&\mathbb{E}\ln \frac{p(\bm{f}_{k,M})}{q(\bm{f}_{k,M})}=\frac{1}{2}\ln\frac{|\Sigma_k|}{|\kappa_{k,MM}|}+\frac{m}{2}\nonumber\\
&\quad-\frac{1}{2}tr\Big(\kappa_{k,MM}^{-1}(\Sigma_k+(\mu_k-g)(\mu_k-g)^T)\Big),\\
&\mathbb{E}[\ln f^2_k(t_n^d)]=-G(-\frac{\tilde{u}_{k,n}^2}{2\tilde{B}_{k,nn}})-C+\ln(\frac{\tilde{B}_{k,nn}}{2}), \\
&\int_{\mathcal{T}}\mathbb{E}[f_k^2(s)]ds=\gamma|\mathcal{T}|-tr(\kappa_{k,MM}^{-1}\Psi_k)\nonumber\\
&\quad+tr(\kappa_{k,MM}^{-1}\Psi_k \kappa_{k,MM}^{-1}(\Sigma_k+\mu_k\mu_k^T)),
\end{align}
$G(x),x\leq 0$ is calculated by a precomputed multi-resolution look-up table. $C$ is a constant and $\Psi_k\in\mathbb{R}^{M\times M}, \Psi_{k,ij}=\int_{\mathcal{T}} \kappa_k(t_i,x)\kappa_k(x,t_j)dx$. $\Psi_k$ is determined by the kernel hyper-parameter in $\kappa_k$ and the region $\mathcal{T}$.

The expectation with regard to beta distribution is:
\begin{align*}
\mathbb{E}[\ln(1-\theta_{dk}')]&=\psi(\tau_{dk,1})-\psi(\tau_{dk,0}+\tau_{dk,1}),\\
\mathbb{E}[\ln(\theta_{dk}')]&=\psi(\tau_{dk,0})-\psi(\tau_{dk,0}+\tau_{dk,1}).
\end{align*}

After adding augmented Lagrangian penalty function, the modified evidence lower bound is:
\begin{align}
L_{\bm{v_i}}(\Phi,\bm{w_i}) &\stackrel{\Delta}{=} \mathcal{L}_1(q)
-\sum_{k=1}^{K}w_{ik}\Big(\int_\mathcal{T}\mathbb{E}_q[f_k^2(s)]ds-A\Big)\nonumber\\
&-\sum_{k=1}^{K}\frac{v_{ik}}{2}\Big(\int_\mathcal{T}\mathbb{E}_q[f_k^2(s)]ds-A\Big)^2.
\label{Equ:FinalBound}
\end{align} 

\subsection{Details of Derivatives}
Based on the modified evidence lower bound in Equation \eqref{Equ:FinalBound}, we could derive the parameter learning method.

\begin{itemize}
	\item $\eta_d$. We list the term related to $\eta_d$ in Equation \eqref{Equ:FinalBound} first.
	\begin{align*}
	L_{\eta_d}&\stackrel{\Delta}{=}N_d\ln\eta_d-\eta_d\int_{\mathcal{T}}\sum_{k=1}^{K}\mathbb{E}\Big(\theta_{dk}f_k^2(s)\Big)ds\\
	&-\eta_db_0+(a_0-1)\ln\eta_d.
	\end{align*}
	Obviously, there is a closed form update for $\eta_d$
	\begin{equation*}
	\eta_d = \frac{N_d+a_0-1}{b_0+\int_{\mathcal{T}}\sum_{k=1}^{K}\mathbb{E}\Big(\theta_{dk}f_k^2(s)\Big)ds}.
	\end{equation*}
	\item $\tau_{dk,0},\tau_{dk,1}$. We list the term related to these parameters in Equation \eqref{Equ:FinalBound} first
	\begin{align*}
	L_{\tau_{dk}}&\stackrel{\Delta}{=}\sum_{n=1}^{N_d}\Big[\ln\sum_{k=1}^{\infty}\exp\Big(\mathbb{E}_q[\ln \theta_{dk}]\\
	&-\mathbb{E}_q[\ln f^2_k(t_n^d)]\Big)\Big]-\eta_d\int_{\mathcal{T}}\mathbb{E}\sum_{k=1}^{\infty}\theta_{dk}f_k^2(s)ds\\
	&+\Big(\ln\frac{\Gamma(\tau_{dk,0})\Gamma(\tau_{dk,1})}{\Gamma(\tau_{dk,0}+\tau_{dk,1})}
	-(\tau_{dk,0}-1)\mathbb{E}\ln\theta_{dk}'\\
	&+(\alpha-\tau_{dk,1})\mathbb{E}\ln(1-\theta_{dk}')\Big).\\
	\end{align*}
	Let
	\begin{align*}
	L_{dnk}&\stackrel{\Delta}{=}\exp\Big(\mathbb{E}_q[\ln \theta_{dk}]+\mathbb{E}_q[\ln f^2_k(t_n^d)]\Big)\\
	&=\exp\Big(\psi(\tau_{dk,0})+\sum_{l=1}^{k-1}\psi(\tau_{dl,1})\\
	&-\sum_{l=1}^{k}\psi(\tau_{dl,0}+\tau_{dl,1})+\mathbb{E}_q[\ln f^2_k(t_n^d)]\Big),\\
	V_k&\stackrel{\Delta}{=}\int_\mathcal{T}\mathbb{E}f^2_k(s)ds
	\end{align*}
	There is no closed form update for these variables, we use coordinate ascent method.
	\begin{align*}
	\frac{\partial L_{\tau_{dk}}}{\partial \tau_{dk,0}}&=-\eta_d\Big(V_k\frac{\partial[\theta_{dk}]}{\partial \tau_{dk,0}}+\sum_{l=k+1}^{K}V_l\frac{\partial[\theta_{dl}]}{\partial \tau_{dk,0}}\Big)\\
	&-\Big(\tau_{dk,0}-1-\sum_{n=1}^{N_d}\frac{L_{dnk}}{\sum_{v=1}^{K}L_{dnv}}\Big)\psi'(\tau_{dk,0})\\
	&+\Big(\tau_{dk,0}-1+\tau_{dk,1}-\alpha-\sum_{n=1}^{N_d}\frac{\sum_{v=k}^{K}L_{dnv}}{\sum_{v=1}^{K}L_{dnv}}\Big)\\
	&\times\psi'(\tau_{dk,0}+\tau_{dk,1}),\\
	\frac{\partial L_{\tau_{dk}}}{\partial \tau_{dk,1}}&=-\eta_d\Big(V_k\frac{\partial[\theta_{dk}]}{\partial \tau_{dk,1}}+\sum_{l=k+1}^{K}V_l\frac{\partial[\theta_{dl}]}{\partial \tau_{dk,1}}\Big)\\
	&-\Big(\tau_{dk,1}-\alpha-\sum_{n=1}^{N_d}\frac{\sum_{v=k+1}^{K}L_{dnv}}{\sum_{v=1}^{K}L_{dnv}}\Big)\psi'(\tau_{dk,1})\\
	&+\Big(\tau_{dk,0}-1+\tau_{dk,1}-\alpha-\sum_{n=1}^{N_d}\frac{\sum_{v=k}^{K}L_{dnv}}{\sum_{v=1}^{K}L_{dnv}}\Big)\\
	&\times\psi'(\tau_{dk,0}+\tau_{dk,1}).
	\end{align*}
	where we have
	\begin{align*}
	&\frac{\partial[\theta_{dk}]}{\partial \tau_{dk,0}}=\frac{\tau_{dk,1}}{(\tau_{dk,0}+\tau_{dk,1})^2}\prod_{l=1}^{k-1}\frac{\tau_{dl,1}}{\tau_{dl,0}+\tau_{dl,1}},\\
	&\frac{\partial[\theta_{dk}]}{\partial \tau_{dk,1}}=-\frac{\tau_{dk,0}}{(\tau_{dk,0}+\tau_{dk,1})^2}\prod_{l=1}^{k-1}\frac{\tau_{dl,1}}{\tau_{dl,0}+\tau_{dl,1}},\\
	&\frac{\partial[\theta_{dl}]}{\partial \tau_{dk,0}}=-\frac{\tau_{dl,0}}{\tau_{dl,0}+\tau_{dl,1}}\frac{\tau_{dk,1}}{(\tau_{dk,0}+\tau_{dk,1})^2}\\
	&\times\prod_{v=1,v\neq k}^{l-1}\frac{\tau_{dv,1}}{\tau_{dv,0}+\tau_{dv,1}},\\
	&\frac{\partial[\theta_{dl}]}{\partial \tau_{dk,1}}=\frac{\tau_{dl,0}}{\tau_{dl,0}+\tau_{dl,1}}\frac{\tau_{dk,0}}{(\tau_{dk,0}+\tau_{dk,1})^2}\\
	&\times\prod_{v=1,v\neq k}^{l-1}\frac{\tau_{dv,1}}{\tau_{dv,0}+\tau_{dv,1}}.
	\end{align*}
	\item $\{\Sigma_k,\mu_k\}$. Take $\mu_k$ for an example.
	\begin{align*}
	\frac{\partial L_{\phi_k}}{\partial \mu_k}&=\sum_{d=1}^{D}\Big(\sum_{n=1}^{N_d}\frac{1}{\sum_{v=1}^{K}L_{dnv}}\frac{\partial L_{dnk}}{\partial \mu_k}\Big)\\
	&-\Big(w_{ik}+v_{ik}(V_k-A)+\sum_{d=1}^{D}\eta_d\mathbb{E}[\theta_{dk}]\Big)\frac{\partial V_k}{\partial \mu_k}\\
	&+\frac{\partial}{\partial \mu_k}\Big[\frac{1}{2}\ln|\Sigma_k|-\frac{1}{2}\ln|\kappa_{k,MM}|\\
	&-\frac{1}{2}tr\Big(\kappa_{k,MM}^{-1}(\Sigma_k+(\mu_k-g)(\mu_k-g)^T)\Big)\Big].
	\end{align*}
\end{itemize}

\textbf{Hyper-parameter part}: We could update the hyper-parameters in a similar way.
\begin{itemize}
	\item Gaussian process hyper-parameters 
	$\kappa_{k,MM},\sigma$. Similar to that in $\{\Sigma_k,\mu_k\}$.
	\item Beta distribution prior $\alpha$.
	\begin{align*}
	L_{\alpha}&\stackrel{\Delta}{=}D(K-1)\ln \alpha+(\alpha-1) \sum_{d=1}^{D}\sum_{k=1}^{K-1}(\psi(\tau_{dk,1})\\
	&-\psi(\tau_{dk,0}+\tau_{dk,1})).
	\end{align*}
	Then we have a closed form update for $\alpha$.
	\begin{equation}
	\alpha=\frac{D(K-1)}{\sum_{d=1}^{D}\sum_{k=1}^{K-1}\Big(\psi(\tau_{dk,1}+\tau_{dk,0})-\psi(\tau_{dk,1})\Big)}\label{Equ:Alpha}.
	\end{equation}
\end{itemize}

\subsection{Proof of Upper Bound}
\begin{theorem}
	Each optimization problem is upper bounded.
	\begin{equation*}
	L_{\bm{v_i}}(\Phi,\bm{w_i})\leq \ln p(Y)+\sum_{k=1}^{K}\frac{w_{ik}^2}{2v_{ik}},\quad i\in\mathbb{N}^+.
	\end{equation*}
	\label{Theorem 1}
\end{theorem}
\begin{proof}
	$\mathcal{L}_1(q)$ can be easily bounded by variational inference framework
	\begin{equation*}
	\mathcal{L}_1(q)\leq \ln p(Y)
	\end{equation*}
	Let $h_{ik}=\int_\mathcal{T}\mathbb{E}_q[f_k^2(s)]ds-A$, and then we have
	\begin{align*}
	&\sum_{k=1}^{K}w_{ik}\Big(\int_\mathcal{T}\mathbb{E}_q[f_k^2(s)]ds-A\Big)\\
	&+\sum_{k=1}^{K}\frac{v_{ik}}{2}\Big(\int_\mathcal{T}\mathbb{E}_q[f_k^2(s)]ds-A\Big)^2\\
	=&\sum_{k=1}^{K}(w_{ik}h_{ik}+\frac{v_{ik}}{2}h_{ik}^2)\geq \sum_{k=1}^{K}\frac{w_{ik}^2}{2v_{ik}}
	\end{align*}
	Combining these two parts finishes the proof.
\end{proof}

\subsection{A Bias When Using Lemma \ref{lem:1}}
Although the bound in Lemma \ref{lem:1} is rather tight, it can still add a bias which may lead to the over-shrinking phenomenon in the model. We illustrate the bias through the following simple model. 
\begin{equation*}
Y_1=X_1^2,~Y_2=X_2^2,~X_1\sim \mathcal{N}(2,1),~X_2\sim \mathcal{N}(2,4),
\end{equation*}
where $\mathcal{N}(\cdot)$ is the normal distribution. Using Lemma \ref{lem:1}, we can arrive the following inequality:
\begin{align}
\mathcal{L}_{left}&\stackrel{\Delta}{=}\mathbb{E}_{p(Y_{1:2})}\ln\Big(wY_1+(1-w)Y_2\Big)\\\nonumber
&\geq \ln\Big(w\exp(\mathbb{E}\ln Y_1)+(1-w)\exp(\mathbb{E}\ln Y_2)\Big)\\
& \stackrel{\Delta}{=}\mathcal{L}_{right},~w\in[0,1]\nonumber.
\end{align}

We vary the value of $w$ and plot $\mathcal{L}_{right}$ and $\mathcal{L}_{left}$. The result is given in Figure \ref{Fig:LowerBias}. We can see that the for $\mathcal{L}_{right}$ the optimal is $w=1$ while for $\mathcal{L}_{left}$ the optimal is obviously a mixture of two components. This is because the logarithm function will punish values which are closer to zero harder. Since $Y_2$ has a large variance, there will be a large proportion of samples near zero which makes the corresponding $\mathbb{E}\ln Y_2$ smaller and less favorable. This bias in the inequality may account for the shrinkage in LPPA and BaNPPA. 

\begin{figure}[ht]
	\includegraphics[width=\columnwidth]{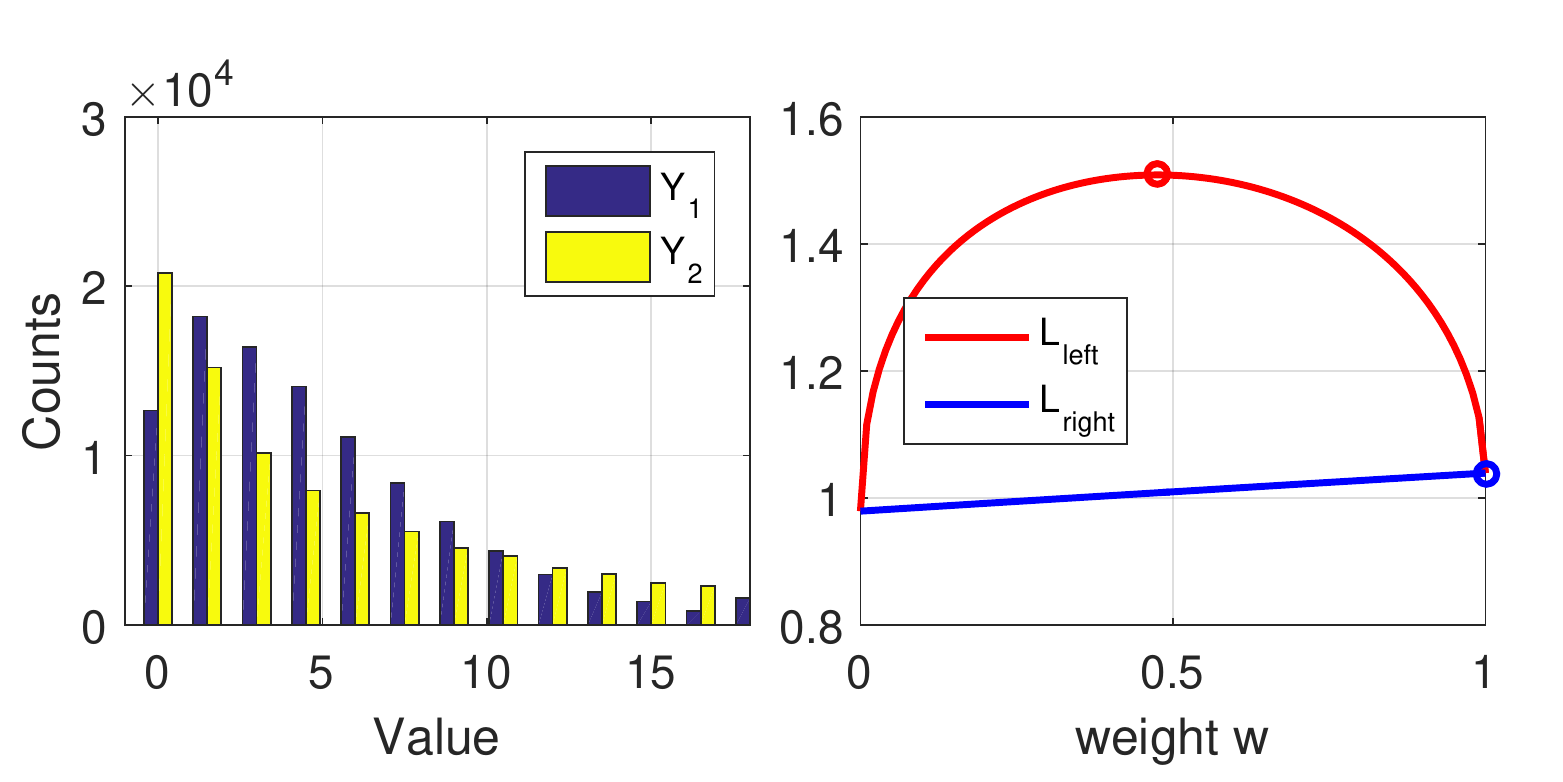}
	\caption{Bias in the inference with lower bound. Left: The histogram of $Y_1$ and $Y_2$. Right: $\mathcal{L}_{left}$ (Blue) versus $\mathcal{L}_{right}$ (Red) and the round marker indicates the maximum of the curve.}
	\label{Fig:LowerBias}
\end{figure}

\section{Test Likelihood}
In LPPA, the allocation matrix $\Theta$ is treated as hyper-parameters and all the parameters are $\{\bm{\mu},\bm{\Sigma},\bm{H},\Theta\}$. Let $\Phi = \{\bm{H},\Theta\}$. In variational inference we use the variational distribution $q(\bm{f};\Phi)$ to approximate the posterior $p(\bm{f}|Y_{train};\Phi)$. The test likelihood can be lower-bounded as follows.
\begin{align}
&\ln p(Y_{test}|Y_{train};\Phi) = \ln \int p(Y_{test}|\bm{f};\Phi)p(\bm{f}|Y_{train};\Phi) d\bm{f} \nonumber\\
& \approx \ln \int p(Y_{test}|\bm{f};\Phi)q(\bm{f};\Phi) d\bm{f} \nonumber\\
& \geq \int q(\bm{f};\Phi)\ln\frac{p(Y_{test}|\bm{f};\Phi)q(\bm{f};\Phi)}{q(\bm{f};\Phi)} d\bm{f} \nonumber\\
& = \mathbb{E}_q \ln p(Y_{test}|\bm{f};\Phi)\nonumber\\
&\geq \sum_{d=1}^{D}\sum_{n=1}^{N_d^{\mathrm{test}}}\ln \sum_{k=1}^{K}\theta_{dk}\exp\Big[\mathbb{E}_q (\ln f_k^2(t_n^d))\Big]\nonumber\\
&-\sum_{d=1}^{D}\sum_{k=1}^{K}\theta_{dk}\int_\mathcal{T}\mathbb{E}_q[f_k^2(s)]ds\stackrel{\Delta}{=}\mathcal{L}_{test}.
\label{Equ: Test}
\end{align} 
In BaNPPA, all the parameters to be optimized are $\{\bm{\eta},\bm{\tau},\bm{\mu},\bm{\Sigma},\bm{H},a_0,b_0,\alpha\}$. Let $\Phi = \{\bm{H},a_0,b_0,\alpha\}$. However, if we follow the same deduction as LPPA, we will not arrive at a fair comparison since the inequality in Equation \eqref{Equ: Test} is different in principle for LPPA and BaNPPA, and therefore, we draw $L$ samples from variational distribution $q(\bm{s},\theta_d;a_0,b_0,\alpha)$ for $\bm{s},\theta_d$ and then follow the lower bound in Equation \eqref{Equ: Test}.
\begin{align}
&\mathbb{E}_q \ln p(Y_{test}|\bm{s},\bm{\Theta},\bm{f};\Phi) \nonumber\\
&= \int q(\bm{s},\bm{\Theta},\bm{f};\Phi) \ln p(Y_{test}|\bm{s},\bm{\Theta},\bm{f};\Phi) d\bm{s}d\bm{\Theta}d\bm{f}\nonumber\\
& \approx \frac{1}{\tilde{L}}\sum_{l=1}^{\tilde{L}}\int q(\bm{f};H) \ln p(Y_{test}|\bm{s}_l,\bm{\Theta}_l,\bm{f};H)d\bm{f}\nonumber\\
& \geq \frac{1}{\tilde{L}}\sum_{l=1}^{\tilde{L}}\Big(\sum_{d=1}^{D}\sum_{n=1}^{N_d^{\mathrm{test}}}\ln\Big(s_{l,d}\sum_{k=1}^{K}\theta_{l,dk}\exp\Big[\mathbb{E}_q (\ln f_k^2(t_n^d))\Big]\Big)\nonumber\\
&-\sum_{d=1}^{D}s_{l,d}\sum_{k=1}^{K}\theta_{l,dk}\int_\mathcal{T}\mathbb{E}_q[f_k^2(s)]ds\Big)\label{Equ:TestBaNPPA}.
\end{align}

\begin{figure*}[ht]
	\centering
	\includegraphics[width=\textwidth]{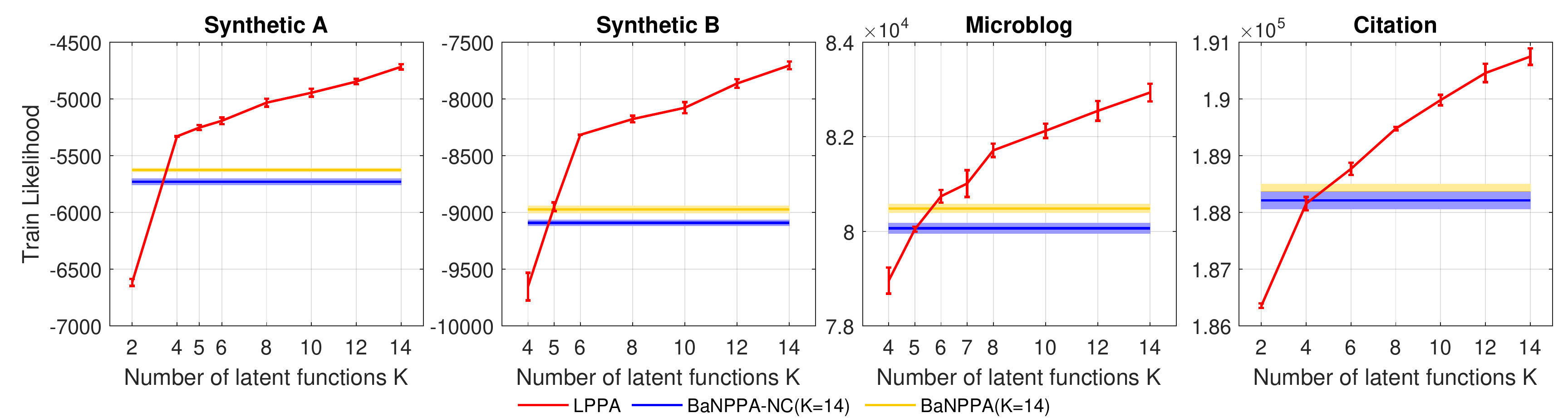}	
	\caption{The comparison of the train likelihood for three algorithms. For LPPA, we change the number of latent functions $K$. For BaNPPA/BaNPPA-NC, we fix $K=14$ and optimize the hyper-parameter $\alpha$ using the variational expectation-maximization. Error bars and shaded area represent the 95\% confidence intervals.}\label{fig:Comparison_Train}
\end{figure*}

\section{Additional Experiment Results}
\subsection{Details of the Data Sets}
\begin{itemize}
	\item {\textbf{Synthetic dataset}}. \\
	A) In $\lambda_d(t)=s_d\sum_{k=1}^{4}\theta_{dk}\tilde{f}(t;\psi_k),~t\in[0,60]$. 
	\begin{align*}
	s_d&\sim \mathrm{Gamma}(2,3),\\
	\theta_d&\sim \mathrm{Dirichlet}(1.2,1,0.8,0.6),\\
	\tilde{f}(t;\psi_k)&\propto\exp\Big(-\frac{(t-15+10k)^2}{10}\Big)\\
	&+\exp\Big(-\frac{(t-55+10k)^2}{10}\Big).
	\end{align*} 
	Each $\tilde{f}(t;\psi_k)$ is either a Gaussian distribution or a mixture of two Gaussian distributions normalized by its integral. 
	
	B) In $\lambda_d(t)=s_d\sum_{k=1}^{6}\theta_{dk}f_k(t)$. 
	\begin{align*}
	s_d&\sim \mathrm{Gamma}(2,3),\\
	\theta_d&\sim \mathrm{Dirichlet}(1.2,1,0.8,0.6,0.5,0.5),\\
	\tilde{f}(t;\psi_k)&\propto\exp\Big(-\frac{(t-15+10k)^2}{10}\Big)\\
	&+\exp\Big(-\frac{(t-75+10k)^2}{10}\Big).
	\end{align*} 
	Each $\tilde{f}(t;\psi_k)$ is either a Gaussian distribution or a mixture of two Gaussian distributions normalized by its integral. We use the rejection sampling method for the inhomogeneous Poisson process to generate the time sequences.
	\item {\textbf{citation dataset}}. Two examples with different citation patterns are given in Figure \ref{fig:EgCitation}.
	\begin{figure}[ht]
		\begin{center}
			\includegraphics[width=0.9\columnwidth]{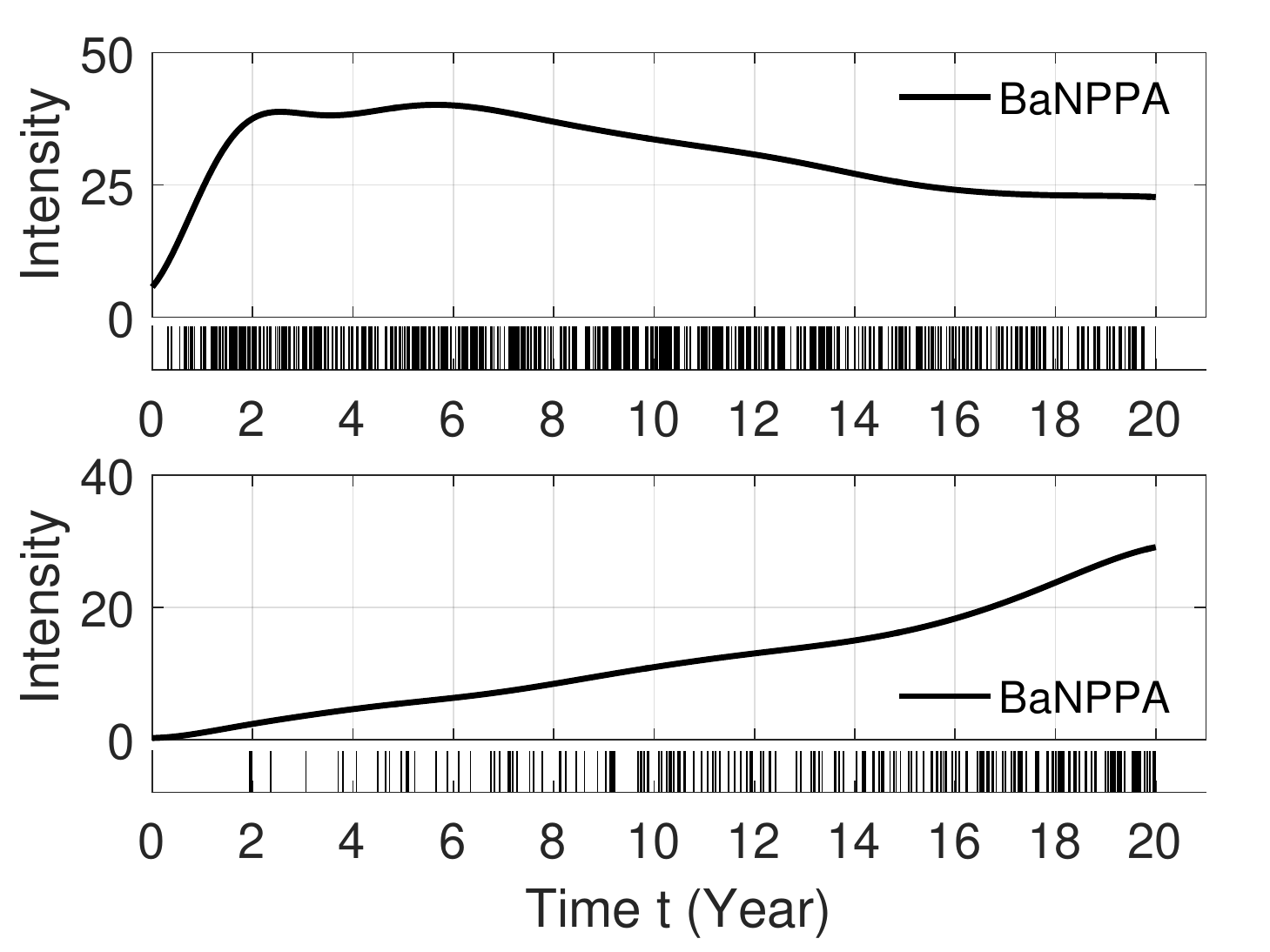}
			\caption{\textbf{Citation data set}. Top: A paper which slowly gets citation and becomes popular many years later.  Bottom: A paper which quickly gets citation after being published. Smooth lines are the mean intensity function inferred from LPPA and BaNPPA. Small bars is the time of each citation. The x-axis indicates the time in year after publication.}\label{fig:EgCitation}
		\end{center}
	\end{figure}
\end{itemize}

\begin{figure*}[ht]
	\centering
	\includegraphics[width=\textwidth]{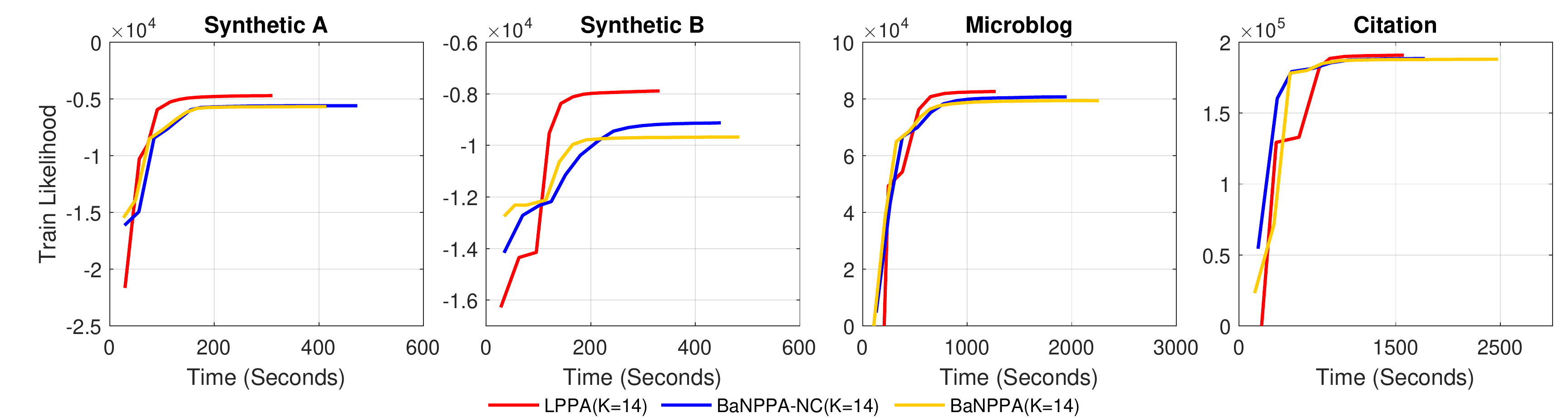}
	\caption{The comparison of the training likelihood versus time for four data sets (K=14) when optimizing the hyper-parameter $\alpha$. The result of one trial is shown.}\label{fig:TrainTime}
\end{figure*}
\begin{figure*}[ht]
	\centering
	\includegraphics[width=\textwidth]{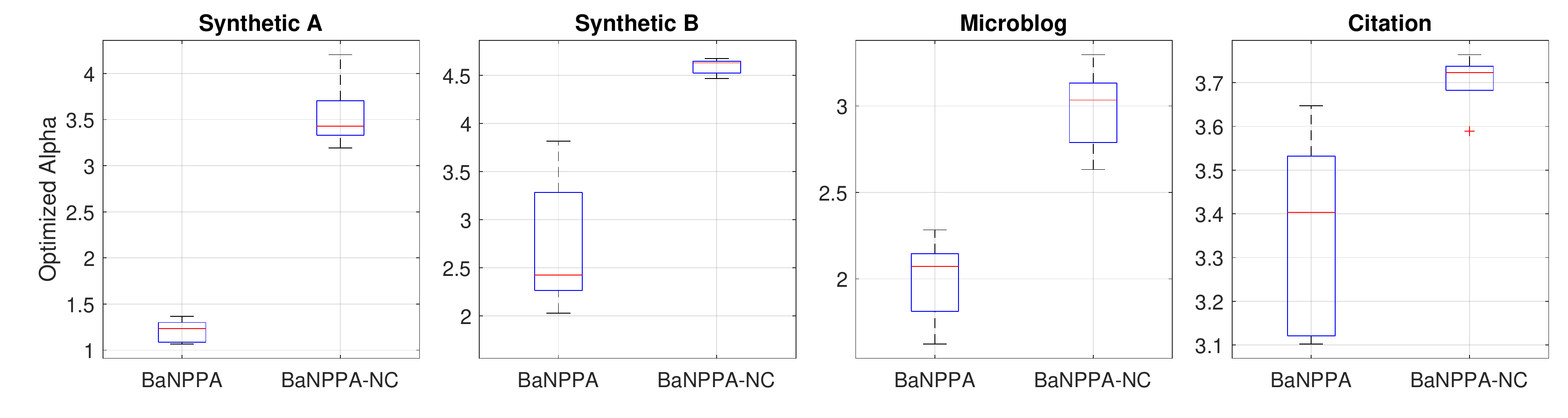}
	\caption{The comparison of the optimized $\alpha$ for four data sets (K=14) when optimizing the hyper-parameter $\alpha$ .}\label{fig:Alpha}
\end{figure*}
\begin{figure*}[ht]
	\centering
	\includegraphics[width=\textwidth]{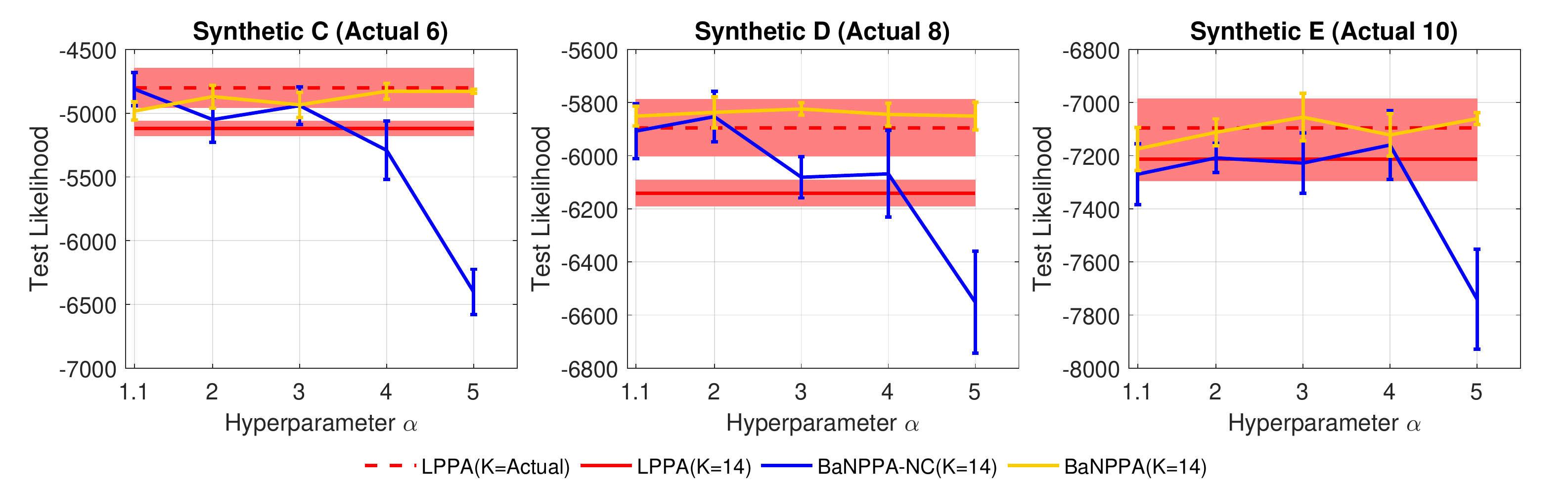}
	\caption{The comparison of the test likelihood for three additional data sets (K=14) when fixing the hyper-parameter $\alpha=[1.1,2,4,6,8]$. Error bars and shaded area represent the 95\% confidence intervals.}
	\label{fig:Comparison_AlphaVary2}
\end{figure*}

\subsection{The Comparison of the Train Likelihood}
The comparison of the train likelihood $\mathcal{L}_{train}$ is given in Figure \ref{fig:Comparison_Train}. We can notice that for LPPA, the train likelihood keeps increasing when we increase $K$. This is also a sign of over-fitting.

\subsection{Computation Time}
We plot the change of the training likelihood in one trial in Figure \ref{fig:TrainTime}. For total computational complexity, both BaNPPA-NC and BaNPPA take more computation time but are still comparable to LPPA. Two reasons account for this fact. One is that there are more parameters to be optimized in BaNPPA and BaNPPA-NC and the other is that BaNPPA potentially has an infinite number of problems to be solved. In Figure \ref{fig:TrainTime}, we can notice that the training likelihood for BaNPPA and the training likelihood for BaNPPA-NC stabilize rather quickly. This is because we use Equation \eqref{Equ:TestBaNPPA} to calculate the likelihood and there are no divergence terms in it.

\subsection{Synthetic Data Sets with a Relatively Large K}
We add three more synthetic data set with a larger K.
\begin{enumerate}
	\item[C)] We sample 200 sequences from $\lambda_d(t)=s_d\sum_{k=1}^{6}\theta_{dk}\tilde{f}(t;\psi_k)$, where $s_d,~\theta_d$ are drawn from Dirichlet distribution and Gamma distribution.
	\begin{align*}
	s_d&\sim\mathrm{Gamma}(2,3),\\
	\bm{\theta}_{d}&\sim\mathrm{Dir}(0.8,0.4,0.2,0.2,0.2,0.2).
	\end{align*}
	We use $\tilde{f}(t;\psi_k)=\exp(-(t-15+10k)^2/10)$, $k=1,\ldots,6$, $t\in[0,60]$ as basis intensity functions. 
	\item[D)] We sample 200 sequences from $\lambda_d(t)=s_d\sum_{k=1}^{8}\tilde{f}(t;\psi_k)$, where $s_d,~\theta_d$ are drawn from Dirichlet distribution and Gamma distribution.
	\begin{align*}
	s_d&\sim\mathrm{Gamma}(2,3),\\
	\bm{\theta}_{d}&\sim\mathrm{Dir}(0.8,0.4,0.4,0.2,0.2,0.2,0.1,0.1).
	\end{align*}
	We use $\tilde{f}(t;\psi_k)\propto\exp(-(t-15+10k)^2/10)$, $k=1,\ldots,8$, $t\in[0,80]$ as basis intensity functions. 
	\item[E)] We sample 200 sequences from $\lambda_d(t)=s_d\sum_{k=1}^{10}\tilde{f}(t;\psi_k)$, where $s_d,~\theta_d$ are drawn from Dirichlet distribution and Gamma distribution.
	\begin{align*}
	s_d&\sim\mathrm{Gamma}(2,3),\\
	\bm{\theta}_{d}&\sim\mathrm{Dir}(0.8,0.6,0.4,0.4,0.4,0.2,0.2,0.2,0.1,0.1).
	\end{align*}
	We use $\tilde{f}(t;\psi_k)\propto\exp(-(t-15+10k)^2/10)$, $k=1,\ldots,10$, $t\in[0,100]$ as basis intensity functions. 
\end{enumerate}
In the experiment, we fix the hyper-parameter $a_0$ and $b_0$ and the length-scale hyper-parameters in all $\kappa_{k,MM}$ to be 4.3081 (Close to the half of the span of $\tilde{f}(t;\psi_k)$). This means we only optimize the mixture weights and the variational distribution $q(m,S)$ for Gaussian processes.

We vary the hyper-parameter $\alpha=[1.1,2,3,4,5]$. The result is given in Figure. We can see that BaNPPA-NC tends to over-shrink the components even when $\alpha=5$ and gets a worse result.

\end{document}

%% file: intro.tex
The Internet age has made it possible to collect a huge amount of temporal data available in the form of time-sequences. Each time-sequence consists of time-stamps which record the arrival times of events, e.g., postings of tweets on Twitter or announcements of life events on Facebook.  In real-world problems arising in areas such as social science \cite{gao2015modeling}, health care \cite{lian2015multitask} and crime prevention \cite{liu2003criminal}, time-sequence modeling is extremely useful since it can help us in predicting future events and understanding the reasons behind them. 


When modeling a collection of time-sequences, a key idea is to cluster the data into groups while allowing the groups to remain linked to share statistical strengths among them~\cite{teh2005sharing}. Several models have been proposed on the basis of this simple idea, e.g., the convolution process \cite{gunter2014efficient}, nonnegative matrix factorization (NMF) \cite{miller2014factorized}, and latent Poisson process allocation (LPPA) \cite{lloyd2016latent}. These models employ latent factors to share statistical strengths and combine these functions to model the correlations within and among time-sequences.



Among these models, LPPA is a powerful approach because it uses latent functions obtained from a Gaussian process (GP). Such continuous latent functions are able to flexibly model complex structures in the data, and do not require a careful discretization such as that used in NMF.
%
%
However, a limitation of LPPA is that the number of latent functions needs to be set beforehand. If the chosen number is much larger than the actual number of latent functions required to explain the data, LPPA will still use all the latent functions. There is no mechanism in LPPA to prevent this ``spread" of allocation, which creates a problem when our goal is to understand the reasons behind the events observed in the data. For example, this might make it difficult to explain the retweet patterns in Twitter where a sudden avalanche of retweets is quite common \cite{gao2015modeling}. For such cases, LPPA will simply use all its latent functions to explain these spiky patterns.


\begin{figure}[!th]
	\begin{center}
		\includegraphics[width=\columnwidth]{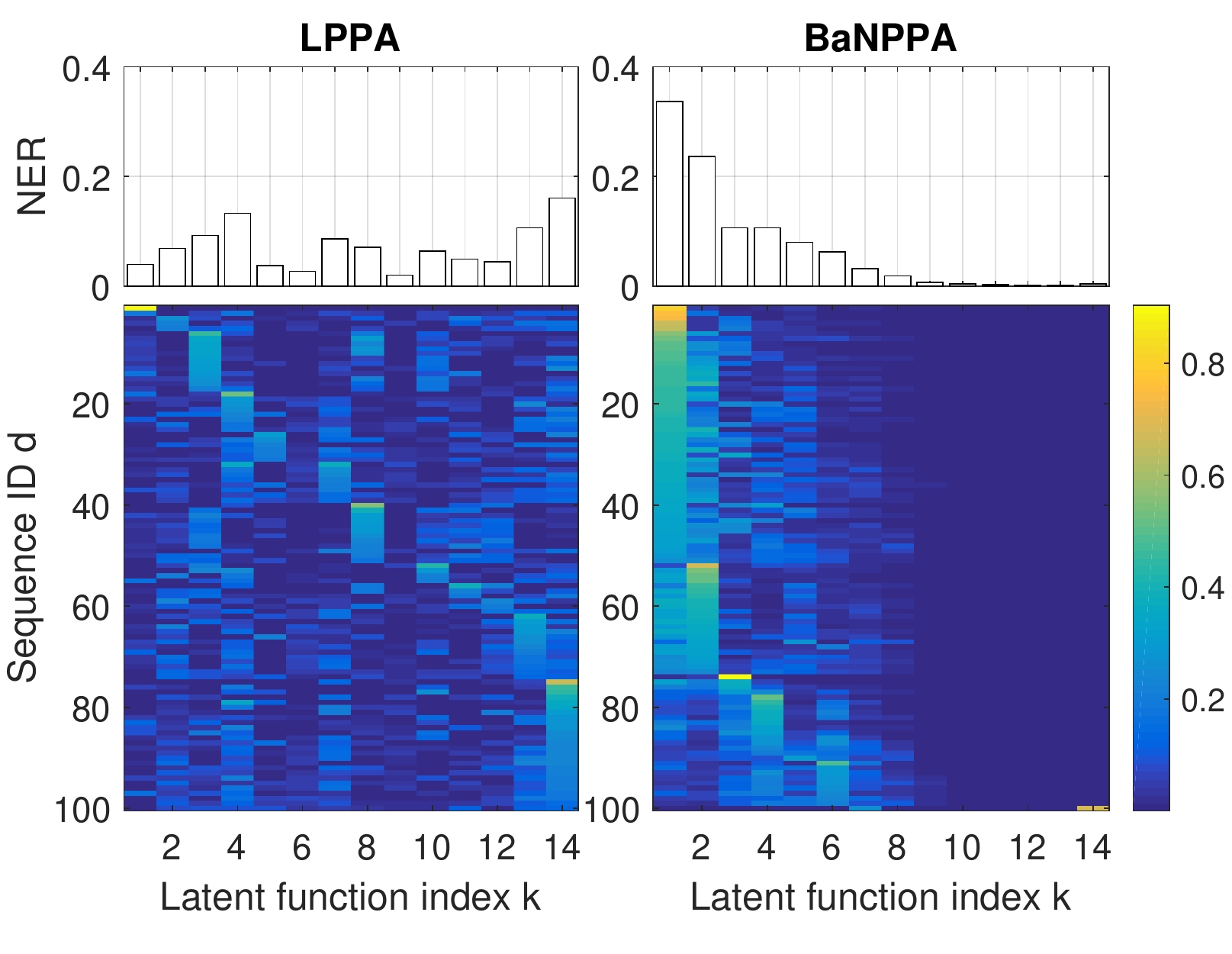}
		\caption{This figure illustrates that, even when a large number of latent functions are provided, BaNPPA automatically selects only a few to explain the data, while LPPA uses them all. The bottom plots show the weights of the latent functions for the Microblog dataset, where we see that BaNPPA assigns zero weights to many latent functions, while LPPA assigns every latent function to at least a few time-sequences. The top plots show a score which measures the average responsibility of the latent functions. See Section \ref{sec:expt} for details.}
		\label{fig:MatrixPlot}
	\end{center}
\end{figure}

In theory, the above problem can be solved by using Bayesian nonparametric (BNP) methods \cite{hjort2010bayesian} which can automatically determine the number of relevant latent functions. However, as we show in this paper, a direct application of existing BNP methods to LPPA is challenging. An obvious issue is that such an application typically requires the use of Markov Chain Monte Carlo (MCMC) algorithms which are slow to converge for large data sets. A more essential and technically intricate issue is that a naive application of BNP methods to LPPA suffers from an unidentifiability issue because the GP-modulated latent functions are not normalized. Unidentifiability is bad news when our focus is to understand the reasons behind the events.

In this paper, we propose a new model to solve these problems. Our model, which we call the \emph{Bayesian nonparametric Poisson process allocation} (BaNPPA) model, enables automatic inference of the number of latent functions while retaining the accuracy, interpretability, and scalability of LPPA.
Unlike hierarchical models \cite{teh2005sharing} which promote sharing through a common base measure, latent functions in our model are shared across all time-sequences due to the size-biased ordering which promotes sharing by penalizing latent functions that belong to higher indices \cite{gopalan2014bayesian, pitman2015size}. The size-biased ordering restricts the use of all latent functions. 
Figure \ref{fig:MatrixPlot} illustrates this on a real data set.

We propose a computationally efficient variational inference algorithm for BaNPPA and solve the unidentifiability issue by adding a constraint within the inference algorithm to regulate the volume of each latent function. Overall, we present a scalable and accurate Bayesian nonparameteric approach for time-sequence modeling. Figure \ref{fig:EgMicroblog} shows an example of the results obtained with BaNPPA on a real-world dataset.

\begin{figure}[!t]
	\begin{center}
		\includegraphics[width=\columnwidth]{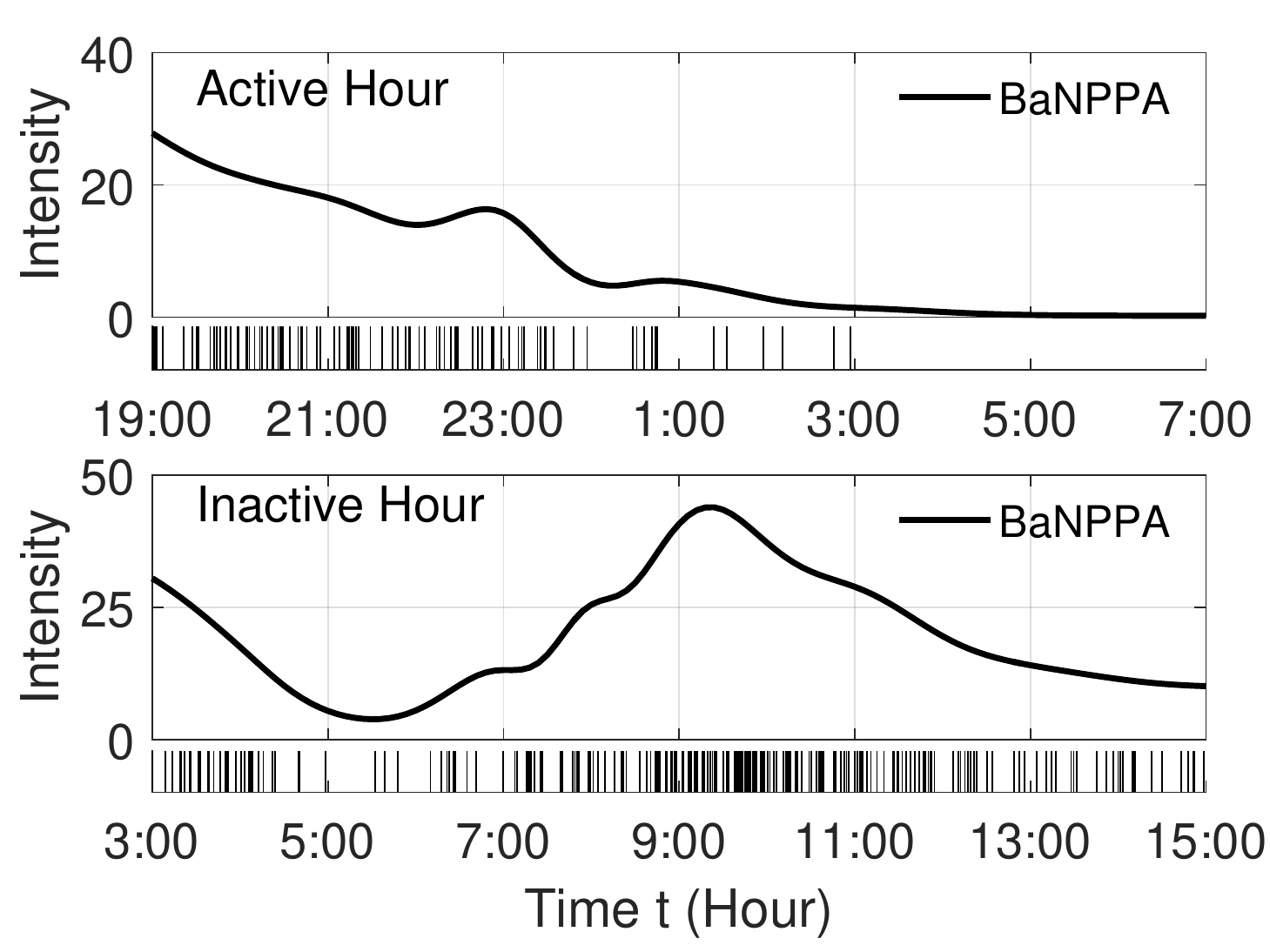}
		\caption{Illustrations of intensity functions obtained with BaNPPA on the Microblog dataset. Each plot shows a time-sequence (with small bars at the bottom) and the corresponding estimated intensity function (with solid lines). The top and bottom plots are for tweets posted during active and inactive hours of the day, respectively.   }\label{fig:EgMicroblog}
	\end{center}
\end{figure}

%% file: prelim.tex
Our goal is to develop a flexible model for time-sequences. We consider time-sequence that contain a set of time-stamps which record the occurrence of events. We denote a time-sequence by $\mathbf{y}_d = \{t^d_n\in \mathcal{T}\}_{n=1}^{N_d}$, where $t_n^d$ is the $n$'th time-stamp in the $d$'th time-sequence,   $\mathcal{T}\subset \mathbb{R}^+$ is a specified time window, and $N_d$ is the number of events. The set of $D$ time-sequences is denoted by $Y$.



A common approach to model such time-sequences is to use the temporal Cox process \cite{adams2009tractable,lloyd2015variational} which uses a stochastic intensity function $\lambda(t):\mathbb{R}^+\rightarrow \mathbb{R}^+$ to model the arrival times~\cite{kingman1993poisson}. Given the intensity function $\lambda(t)$ and a time window $\mathcal{T}\subset \mathbb{R}^+$, the number of events $N(\mathcal{T})$ is Poisson distributed with rate parameter $\int_{\mathcal{T}}\lambda(s)ds$. Therefore, the likelihood of a sequence $\mathbf{y}_d$ drawn from the temporal Cox process is equal to: 
\begin{equation}
	\mathrm{P}(\mathbf{y}_d|\lambda_d)=\exp\Big(-\int_{\mathcal{T}}\lambda_d(s)ds\Big)\prod_{n=1}^{N_d}\lambda_d(t_n^d).
	\label{Equ:PP}
\end{equation}
In LPPA, to model multiple time-sequences, the $d$'th time-sequence is assumed to be generated by a temporal Cox process with an intensity function $\lambda_d(t)$ which is modeled as follows:
\begin{equation}
	\lambda_d(t)=\sum_{k=1}^{K}\theta_{dk}f_k^2(t),\quad \theta_{dk}\geq 0,
	\label{Equ:lambda}
\end{equation}
where $f_k(t)$ is a function drawn from a GP prior, $\theta_{dk}$ is its weight, and $K$ is the number of latent functions. To ensure the non-negativity of $\lambda_d$, $f_k$ are squared and weights $\theta_{dk}$ are required to be non-negative. 

LPPA is a powerful approach which also enables scalable inference. Due to the GP prior, LPPA is capable of generating intensity functions with complex shapes. Scalable inference is made possible by using variational inference for sparse GPs \cite{titsias2009variational}.
The overall computational complexity is $O(KNM^2)$, where $N$ is the total number of events in $Y$ and $M$ is the number of pseudo inputs in sparse GPs. 


One issue with LPPA is that $K$ needs to be set beforehand. This not only increases the computation cost, but also creates a serious interpretability issue which is undesirable when our goal is to understand the reasons behind the data. 
Specifically, when the number of latent functions is much larger than what it needs to be, LPPA uses all of them, making it difficult to interpret the results. We give empirical evidence in support of this claim and correct this behavior by using a BNP method.

Unfortunately, a direct application of the existing BNP methods increases the computation cost and limits the flexibility of the model. The problem lies in the strict requirement that the latent functions needs to be a \emph{normalized density function}, i.e., a function with a volume\footnote{The volume of a function $f(t),t\in\mathcal{T}$ is defined as the integral $\int_{\mathcal{T}}f(t)dt$.} equal to 1.
For example, previous studies, such as \citet{kottas2006dirichlet, ihler2007learning}, model the intensity functions with the following Dirichlet process mixture model,
\begin{align}
\lambda_d(t) = s_d \sum_{k=1}^\infty\theta_{dk}\tilde{f}(t;\psi_k) , \label{Equ:lambdaAlter0}
\end{align}
where $\tilde{f}$ are normalized density functions with parameters $\psi_k$ and the weights $\theta_{dk}$ are non-negative and sum to one $\sum_{k=1}^\infty \theta_{dk} = 1$ ($s_{d}>0$ is the rate parameter that models the number of events $N(\mathcal{T})$).
Since the function $\tilde{f}$ needs to be normalized, the choices are limited to well-known density function which may not be very flexible to model complex time-sequences, e.g., \citet{kottas2006dirichlet} used the beta distribution and \citet{ihler2007learning} used the truncated Gaussian distribution.
In addition, such models require MCMC sampling algorithms which usually converge slowly on large data sets.
To the best of our knowledge, it is still unclear how to build a nonparametric prior for such normalized density functions while enabling scalable inference, e.g., via variational methods.

We propose a nonparameteric model, called the Bayesian nonparameteric Poisson process allocation (BaNPPA), which avoids the need to explicitly specify the number of latent functions while retaining the flexibility and scalability of the LPPA model. Our method combines the models shown in Equation \eqref{Equ:lambda} and \eqref{Equ:lambdaAlter0}. We show that this direct combination has an unidentifiability issue, and we fix the issue within a variational-inference algorithm. Our approach therefore combines the strengths of the LPPA and BNP models while keeping their best features.

%% file: model.tex
As discussed earlier, we need to set the number of latent functions beforehand for LPPA. We fix this issue by proposing a new model called BaNPPA that combines the non-parametric model of Equation \eqref{Equ:lambdaAlter0} with the LPPA model shown in Equation \eqref{Equ:lambda}.
%
%
Specifically, we let $\tilde{f}$ in Equation \eqref{Equ:lambdaAlter0} to be equal to $f^2_k(t)$, as follows: 
\begin{equation}
\lambda_d(t)=s_d\sum_{k=1}^{\infty}\theta_{dk}f^2_k(t), \,\, \textrm{where } s_d,\theta_{dk}>0, \,\, \sum_{k=1}^{\infty}\theta_{dk}=1 .
\label{Equ:lambdaAlter}
\end{equation}
Similar to LPPA, we draw functions $f_k(t)$ from a Gaussian process. We draw the weights $\theta_{dk}$ using a stick-breaking process, and use a Gamma prior for the scalar rate parameter $s_d$.
The final generative model of BaNPPA is shown below:
%
\begin{enumerate}
	\item Draw $f_k\sim \mathrm{GP}(m_k(t),\kappa_k(t,t'))$ for $k=1,\ldots,\infty$.
	\item For each sequence $d=1,\ldots,D$,
	\begin{itemize}
		\item Draw $\theta_{dk}'\sim \mathrm{Beta}(1,\alpha)$ for $k=1,\ldots,\infty$. 
		\item Calculate  $\theta_{dk}=\theta_{dk}'\prod_{l=1}^{k-1}(1-\theta_{dl}')$.
		\item Draw $s_d\sim \mathrm{Gamma}(a_0,b_0)$.
		\item Draw the points $\mathbf{y}_d\sim \mathrm{PP}(s_d\sum_{k=1}^{\infty}\theta_{dk}f_k^2(t))$.
	\end{itemize}
\end{enumerate}
%
%
%
In the model, we denote a Poisson process with rate parameter $\lambda$ by $\mathrm{PP}(\lambda)$, a beta distribution with shape parameters $a$ and $b$ by $\mathrm{Beta}(a,b)$ and a gamma distribution with shape parameter $a$ and rate parameter $b$ by $\mathrm{Gamma}(a,b)$.

The above model automatically determines the number of latent functions due to the size-biased ordering \cite{pitman2015size} obtained by using the stick-breaking process.
Both the latent functions $\{f_k^2(t)\}$ and the weights $\{\theta_{dk}\}$ use the same set of indices $k=1,\ldots,\infty$. This implies that when generating the $d$'th time-sequence, the latent function at a lower index $k$ is more likely to be assigned a larger weight $\theta_{dk}$. This encourages the model to use some latent functions more than the others. 

Unfortunately, the above model is unidentifiable. This is because, unlike the nonparametric model of Equation \eqref{Equ:lambdaAlter0}, the latent functions $\{f_k^2\}$ are unnormalized, and therefore many combinations of $s_d$, $\{\theta_{dk}\}$ and $\{f_k\}$ might give us the same model.
For example, the following transformation gives the same intensity function for any $\epsilon_k >0$: 
\begin{equation}
   s_d\bar{\epsilon}_d,\left\{\frac{\theta_{dk}\epsilon_{k}}{\bar{\epsilon}_d}\right\},\left\{\frac{f_k}{\sqrt{\epsilon_k}}\right\} ,
\end{equation}
where $\bar{\epsilon}_d := \sum_{v=1}^{\infty}\theta_{dv}\epsilon_{v}$.
We can check this by substituting the triplet in Equation \eqref{Equ:lambdaAlter}. Since the volume of each $f_k$ is not regulated, we can move the ``mass" around between the components of the model.

This type of unidentifiability is problematic when our goal is to understand the reasons behind the patterns in the data. 
In our experiments, we observe that this leads to a shrinkage of the latent functions which affects interpretability as well as the quality of the estimated hyperparameters. In Section \ref{Sec:Ident}, we propose a way to fix this issue by adding a constraint on the volume of the latent function. 

There is also another common identifiability problem in such mixture models.
\citet{lloyd2016latent} claimed that LPPA is unidentifiable and non-unique since there may be multiple decompositions that are well supported by the data. In BaNPPA, due to the ordering constraints imposed by size-biased ordering, this unidentifiability issue is reduced. 

We also need to guarantee that the expected intensity function at any time $\mathbb{E}[\lambda_d(t)]$ is finite. This can be achieved by fixing the GP hyperparameters. For example, assuming a constant mean function $m_k(t) \equiv g$ with $g$ being the constant, and an automatic relevance determination (ARD) covariance functions $\kappa_k(t,t')=\gamma_k\exp(-(t-t')^2/(2a_k^2))$, we can fix the hyperparameters $g$ and $\gamma_k$, which ensures that the mean and
variance of each latent function $f_k$ are finite. In that case, the value of $\mathbb{E}[\lambda_d(t)]$ is
bounded due to the following relation: 
\begin{align}
\mathbb{E}[\lambda_d(t)]& = \mathbb{E}\left[s_d\sum_{k=1}^\infty\theta_{dk}f_k^2(t)\right]\leq\mathbb{E}[s_d]\max_k\mathbb{E}[f_k^2(t)]\nonumber\\
&=\frac{a_0}{b_0}\max_k\Big(\mathbb{E}^2[f_k(t)]+\mathrm{Var}[f_k(t)]\Big).
\end{align}